\documentclass{article}

\usepackage{graphicx} 
\usepackage{subfigure}

\usepackage{natbib}

\usepackage{algorithm}
\usepackage{algorithmic}

\usepackage{multirow} 

\usepackage{amsmath,amsthm, amssymb, appendix}

\newtheorem{thm}{Theorem}
\newtheorem{assumption}[thm]{Assumption}
\newtheorem{cor}[thm]{Corollary}
\newtheorem{lem}[thm]{Lemma}
\newtheorem{prop}[thm]{Proposition}
\theoremstyle{definition}
\newtheorem{defn}[thm]{Definition}

\theoremstyle{remark}

\newtheorem*{rem*}{Remark}

\def\sign{\operatorname{sign}}

\def\R{\mathbb R}

\def\sign{\operatorname{sgn}}

\def\P{\mathcal P}
\def\trace{\operatorname{tr}}

\def\nsmall{n_\flat}
\def\Vbig{V_\sharp}
\def\Vsmall{V_\flat}
\def\ellbig{\ell_\sharp}
\def\ellsmall{\ell_\flat}
\def\Psmall{{\mathcal P}_{\flat}}
\def\Pbig{{\mathcal P}_{\sharp}}
\def\idop{{\mathcal Id}}
\def\range{\operatorname{Range}}

\def\sigmamin{\sigma_{\operatorname{min}}}
\def\poly{\operatorname{poly}}
\global\long\def\DD{\mathfrak{D}}
\global\long\def\PP{\mathcal{P}}

\begin{document}

\title{Breaking the Small Cluster Barrier of Graph Clustering}
\author{Nir Ailon\thanks{Department of Computer Science, Technion Israel Institute of Technology. \texttt{nailon@cs.technion.ac.il}}    \and Yudong Chen\thanks{Department of Electrical and Computer Engineering, The University of Texas at Austin. \texttt{ydchen@utexas.edu}} \and  Xu Huan\thanks{Department of Mechanical Engineering, National University of Singapore. \texttt{mpexuh@nus.edu.sg}}}
\maketitle
%
%
%
%
\begin{abstract}

This paper investigates graph clustering in the planted cluster model in the
 presence of  {\em small clusters}. Traditional results dictate that for an
 algorithm to provably correctly recover the clusters, {\em all} clusters must be
 sufficiently large (in particular, $\tilde{\Omega}(\sqrt{n})$ where $n$ is the number
 of nodes of the graph). We show that this is not really a restriction: by a more refined
 analysis of the trace-norm based matrix recovery approach proposed in \citet{Jalali2011clustering} and \citet{chen2012sparseclustering}, we prove that small clusters, under certain mild assumptions, do not hinder recovery of large ones.
Based on this result, we further devise an iterative algorithm
 to recover {\em almost all clusters} via a ``peeling strategy'', i.e., recover large clusters
 first, leading to a reduced problem, and repeat this procedure. 
These results are extended to the 
 {\em partial observation} setting, in which only a (chosen) part of the graph is observed.
The peeling strategy gives rise to an active learning algorithm, in which 
edges adjacent to smaller clusters are queried more often as large clusters are learned
(and removed).

From a high level, this paper sheds novel insights on high-dimensional statistics and
 learning structured data, by presenting a structured matrix learning problem for which
a one shot convex relaxation approach necessarily fails, but a carefully constructed sequence of convex relaxations
does the job.

\end{abstract}


\section{Introduction}
\label{sec:intro}

This paper considers a classic problem in machine learning and theoretical computer science,
namely graph clustering, i.e., given an undirected unweighted graph, partition the nodes into
disjoint clusters, so that the density of edges within one cluster is higher than those across
clusters. Graph clustering arises naturally in many application across science and engineering.
Some prominent examples include community detection in social network \cite{Social1},
submarket identification in E-commerce and sponsored search \cite{Yaho}, and co-authorship
analysis in analyzing document database \cite{db1}, among others.  From a purely binary classification theoretical
point of view, the edges of the graph are  (noisy) labels of \emph{similarity} or \emph{affinity} between pairs of objects, and the concept class consists of clusterings of the objects (encoded graphically by identifying clusters with cliques).

Many theoretical results in graph
clustering~\citep[e.g.,][]{boppana1987eigenvalues,chen2012sparseclustering,mcsherry2001spectralpartitioning} consider the planted partition model,
in which the edges are generated randomly; see Section~\ref{sss.previouswork} for more details. While  numerous
different methods have been proposed,  their performance guarantees  all share the
following manner~--~under certain condition of the density of edges (within clusters and across
clusters), the proposed method succeeds to recover the correct clusters  exactly {\em if all clusters
are larger than a threshold size}, typically $\tilde{\Omega}(\sqrt{n})$.

In this paper, we aim to break this {\bf small cluster barrier} of graph clustering. Correctly identifying
extremely small clusters is inherently hard as they are easily confused with ``fake'' clusters generated
by noisy edges\begin{footnote}{Indeed, even in a
more lenient setup where one clique (i.e., a perfect cluster) of size $K$ is embedded in an Erdos-Renyi
graph of $n$ nodes and $0.5$ probability of forming an edge, to recover this clique, the best known
polynomial method requires $K=\Omega(\sqrt{n})$ and it has been a long standing open problem to
relax this requirement.}\end{footnote}, and is not the focus of this paper. Instead, in this paper we investigate a question that has
not been addressed before: Can we still recover large clusters in the presence of small clusters? Intuitively,
this should be doable. To illustrate, consider an extreme example where the given graph $G$
consists two disjoint subgraphs $G_1$ and $G_2$, where $G_1$ is a graph that can be correctly
clustered using some existing method, and $G_2$ is a small-size clique.  $G$ certainly violates
the minimum cluster size requirement of previous results, but why should $G_2$ spoil our ability to cluster $G_1$?

Our main result confirms this intuition.  We show that the  cluster size barrier arising in
previous work~\citep[e.g.,][]{Chaudhuri,bollobas2004maxcut,chen2012sparseclustering,mcsherry2001spectralpartitioning} is not really a restriction, but rather an artifact of the attempt to solve the
problem in a single shot using convex relaxation techniques.
Using a more careful analysis, we prove that the mixed trace-norm and $\ell_1$ based convex formulation, initially
proposed in \citet{Jalali2011clustering}, can recover clusters of size $\tilde{\Omega}(\sqrt{n})$ even in the presence
of smaller clusters. That is, small clusters do not interfere with recovery of the big clusters.

The main  implication of this result is that one can apply an iterative ``peeling'' strategy, 
recovering smaller and smaller clusters.
The intuition is
simple~--~suppose the {\em number} of clusters is limited, then either all clusters are large, or the sizes of the clusters vary significantly. The first case is obviously easy. The second one is equally easy:
 use the aforementioned convex formulation, the larger clusters can be correctly identified. If
we remove all nodes from these larger clusters, the remaining subgraph contains significantly fewer nodes than
the original graph, which leads to a much lower threshold on the size of the cluster for  correct recovery, making it possible
for correctly clustering some smaller clusters. By repeating this procedure, indeed, we can recover the cluster structure for almost all nodes {\em with no lower bound on the minimal cluster size}.
We summarize our main contributions and techniques: 

{\bf (1)} We provide a refined analysis of the mixed trace norm and $\ell_1$ convex relaxation approach for exact recovery of clusters
proposed in \citet{Jalali2011clustering} and \citet{chen2012sparseclustering}, focusing on the case where small clusters exist. We show
that in the classical  planted partition settings~\cite{boppana1987eigenvalues}, if each cluster
is either large (more precisely, of size at least $x \approx \sqrt{n}\log^2 n$) or small (of size at most $x/\log^2 n$),
then with high probability, this convex relaxation  approach correctly identifies all big clusters while
``ignoring'' the small ones.
Notice that the multiplicative gap between the two thresholds is logarithmic w.r.t.\ $n$. In addition,
it is possible to arbitrarily increase $x$, thus turning a ``knob'' in quest of an interval $(x/\log^2 n, x)$ that is disjoint
from the set of cluster sizes.
The analysis is done by identifying a certain feasible solution to the convex program and proving its almost sure optimality
using a careful construction of a \emph{dual certificate}.   This feasible solution easily identifies the big clusters. This method has been performed before only in the case where
all clusters are of size $\geq x$.


{\bf (2)} We provide a converse of the result just described. More precisely, we show that
if for some value of the knob $x$ an optimal solution appears to look as if the interval  $(x/\log^2 n, x)$ were indeed free of cluster sizes,
then the solution is useful (in the sense that it correctly identifies big clusters) even if this weren't the case.  

{\bf (3)} The last two points imply that if some interval of the form $(x/\log^2 n, x)$ is free of cluster sizes, then an exhaustive
search of this interval will constructively find big clusters (though not necessarily for that particular interval).
This gives rise to an iterative algorithm, using a ``peeling strategy'', to recover smaller and smaller clusters that are
otherwise impossible to recover.  
Using the ``knob'', we prove that as long as the {\em number} of clusters is bounded by
$\Omega(\log n/\log\log n)$, regardless of the cluster sizes, we can correctly recover  the cluster structure
for an overwhelming fraction of nodes. To the best of our knowledge, this is the first result of provably
correct graph clustering without {\em any} assumptions on the cluster sizes.

{\bf (4)} We extend the result to the partial observation case, where only a faction
of similarity labels (i.e., edge/no edge) is known. As expected, smaller observation rates allow identification
of larger clusters.  Hence, the observation rate serves as the ``knob''.
This gives rise to an {\em active learning algorithm}
for graph clustering based on  adaptively increasing the rate of sampling in order to hit a ``forbidden interval''
free of cluster sizes, and concentrating on smaller inputs as we identify big clusters and peel them off.

Beside these technical contributions, this paper provides novel insights into low-rank matrix recovery and
more generally high-dimensional statistics, where data are typically assumed to obey
certain low-dimensional structure.  Numerous methods have been developed to exploit this {\em a priori} information
so that a consistent estimator is possible even when the dimensionality of data is larger than the number of samples. Our result shows that one may combine these methods with a ``peeling strategy'' to further push the envelope
of learning structured data~--~By iteratively recovering the easier structure and then reducing the problem size, it
is possible to learn structures that are otherwise difficult using previous approaches.

\subsection{Previous work}\label{sss.previouswork}
The literature of graph clustering is too vast for a detailed survey here; we concentrate on the most related
work, and in specific those provide theoretical guarantees on cluster recovery.

{\bf Planted partition model:}  The setup we study is the classical {\em planted partition} model~\cite{boppana1987eigenvalues}, also
known as the {\em stochastic block} model~\cite{Rohe10}. Here, $n$ nodes are partitioned into subsets, referred as the
``true clusters'', and a graph is randomly generated as follows: for each pair of nodes, depending on whether they
belong to a same subset, an edge connecting them is generated with a probability $p$ or $q$ respectively. The goal
is to correctly recover the clusters given the random graph. The planted partition model has been studied as early as 1980's~\cite{boppana1987eigenvalues}. Earlier work focused on the 2-partition or more generally $l$-partition case, i.e., the
minimal cluster size is $\Theta(n)$~\cite{boppana1987eigenvalues,CondonKarp,carson2001planted,bollobas2004maxcut}. Recently, several works have proposed
methods to handle sublinear cluster sizes. These
works can be roughly classified into three approaches: randomized algorithms~ \citep[e.g.,][]{shamir2007improved},
spectral clustering~\citep[e.g.,][]{mcsherry2001spectralpartitioning, giesen2005manypartition, Chaudhuri, Rohe10}), and low-rank matrix decomposition~\cite{Jalali2011clustering,chen2012sparseclustering,AMEVAV11,OYMHAS11}. While these work differs in the methodology, they all impose
constraints on the size of the minimum true cluster~--~the best result up-to-date requires
it to be $\tilde{\Omega}(\sqrt{n})$.

{\bf Correlation Clustering}  This problem, originally defined by Bansal, Blum and Chawla \cite{BBC05}, also considers graph clustering
but in an adversarial noise setting.  The goal there is to find the clustering minimizing the total disagreement
(intercluster edges plus intracluster nonedges), without there being necessarily a notion of true clustering (and hence no
``exact recovery'').  This problem is usually studied in the combinatorial optimization framework and is known to be
 NP-Hard to approximate to within some constant factor.  Prominent work includes \citet{DEFI06,Ailon:2008:AII,CharikarGW05}.
A PTAS is known in case the number of clusters is fixed \cite{GiotisGuruswami07}.

{\bf Low rank matrix decomposition via trace norm:} Motivated from robust PCA, it has recently been shown~\cite{cspw,candeswrightma},
that it is possible to recover a low-rank matrix from sparse errors of arbitrary magnitude, where the key ingredient
is using trace norm (aka nuclear norm) as a convex surrogate of the rank. A similar result is also obtained when
the low rank matrix is corrupted by other types of noise~\cite{XuCaramanisSanghaiv12-OP}.

Of particular relevance to this paper is \citet{Jalali2011clustering} and \citet{chen2012sparseclustering}, where the authors apply this
approach to graph clustering, and specifically to the planted partition model. Indeed, \citet{chen2012sparseclustering} achieve state-of-art performance guarantees for the planted partition problem. However, they don't overcome the  $\tilde{\Omega}(\sqrt{n})$ minimal cluster size lower bound.

{\bf Active learning/Active clustering} Another line of work that motivates this paper is study of active learning algorithms
(a settings in which labeled instances are chosen by the learner, rather than by nature), and in particular
active learning for
clustering.
The most related work is  \citet{AilonBE12}, who  investigated active learning for  correlation clustering.
The authors obtain a $(1+\varepsilon)$-approximate solution with respect
to the optimal, while (actively) querying  no more than $O(n\poly(\log n,k,\varepsilon^{-1}))$
edges.  The result imposed no restriction on cluster sizes and hence inspired this work, but differs in at least two major ways.  First, \citet{AilonBE12} did not
consider \emph{exact recovery} as we do.  Second, their guarantees fall in the  ERM (Empirical Risk Minimization)
framework, with no running time guarantees.  Our work recovers true cluster exactly using  a convex relaxation algorithm,
and is hence computationally efficient.
The problem of active learning has also been investigated in
 other clustering setups including clustering based on distance matrix~\cite{VBRTX-activeclustering12,Shamir11budget}, and
hierarchical clustering~\cite{Eriksson,krishnamurthy2012hierarchical}. These setups differ from ours and cannot
be easily compared.



\section{Notation and Setup}

Throughout, $V$ denotes a ground set of elements, which we identify with the set $[n] = \{1,\dots, n\}$.
We assume a true ground truth  clustering of $V$ given by a pairwise disjoint covering $V_1,\dots, V_k$, where $k$ is the number
of clusters. We say $i\sim j$ if $i,j\in V_a$ for some $a\in [k]$,
otherwise $i\not\sim j$.  We let $n_i = |V_i|$ for all $i\in[k]$.  For any $i\in[n]$, $\langle i\rangle$ is the unique
index satisfying $i \in V_{\langle i\rangle}$.

For a matrix $X \in \R^{n\times n}$ and a subset $S\subseteq [n]$ of size $m$, the matrix
$X[S] \in \R^{m\times m}$ is the principal minor of $X$ corresponding to the set of indexes $S$.
For a matrix $M$, $\Gamma(M)$ denotes the support of $M$, namely, the set of index pairs $(i,j)$ such that $M(i,j)\neq 0$.

The ground truth clustering matrix, denoted $K^*$, is defined so that $K^*(i,j)=1$ is $i\sim j$, otherwise $0$.
This is a block diagonal matrix, each block consisting of $1$'s only.  Its rank is $k$.
The input is a symmetric matrix $A$, a noisy version of $K^*$.  It is generated using the well known \emph{planted clustering} model,
as follows.  There are two fixed edge probabilities, $p>q$.
We think of  $A$ as the adjacency matrix of an undirected random graph, where edge $(i,j)$ 
is in the graph for $i>j$ with probability $p$ if $i\sim j$, otherwise with probability $q$, independent of other choices.
The error matrix is denoted by $B^* := A - K^*$.   We let $\Omega := \Gamma(B^*)$ denote the \emph{noise locations}.

Note that our results apply to the more practical case in which the edge probability of $(i,j)$ is $p_{ij}$ for each $i\sim j$ and  $q_{ij}$  for $i\not\sim j$, as long as $(\min p_{ij}) =: p > q := (\max q_{ij})$.



\section{Results}
We remind the reader that the trace norm of a matrix is the sum of its singular values, and we define the $\ell_1$ norm
of a matrix $M$ to be $\|M\|_1 = \sum_{ij}|M(ij)|$.
Consider the following convex program, combining the trace norm of a matrix variable $K$ with the $\ell_1$ norm of
  another matrix variable $B$ using two parameters $c_1,c_2$ that will be determined later:
\begin{eqnarray*}
\mbox{(CP1) }\min &  & \left\Vert K\right\Vert _{*}+c_{1}\left\Vert \PP_{\Gamma(A)}B\right\Vert _{1}+c_{2}\left\Vert \PP_{\Gamma(A)^{c}}B\right\Vert _{1}\\
\mbox{s.t.} &  & K+B=A\\
 &  & 0\le K_{ij}\le1,\forall(i,j).
\end{eqnarray*}

\begin{thm}
\label{thm:main}There exist constants $b_{1},b_{3},b_{4}>0$
such that the following holds with probability at least $1-n^{-3}$.
For any parameter $\kappa \geq 1$ and $t\in [\frac{1}{4}p+\frac{3}{4}q,\frac{3}{4}p+\frac{1}{4}q]$, define 
\begin{eqnarray}\label{ells}
\ellbig = b_{3}\frac{\kappa \sqrt{p(1-q)n}}{p-q}\log^{2}n \ \ 
 \ellsmall = b_{4}\frac{\kappa \sqrt{p(1-q)n}}{p-q}\ .
\end{eqnarray}
If for all $i\in [k]$,  either $n_i \geq \ellbig$ or $n_i \leq \ellsmall$ and if
 $(\hat K, \hat B)$ is  an optimal solution to (CP1), with 
\begin{eqnarray}\label{cees}
c_{1}=\frac{b_{1}}{\kappa \sqrt{n\log n}}\sqrt{\frac{1-t}{t}} \ \ \ 
c_{2}=\frac{b_{1}}{\kappa \sqrt{n\log n}}\sqrt{\frac{t}{1-t}} \ ,
\end{eqnarray}
then $(\hat K, \hat B) = (\Pbig K^*, A - \hat K)$, where for a matrix $M$, $\Pbig M$ is the matrix defined by
$$ (\Pbig M)(i,j) = \begin{cases} M(i,j) & \max\{n_{\langle i\rangle}, n_{\langle j\rangle}\} \geq \ellbig \\ 0 & \mbox{otherwise} \end{cases}\ .
$$
\end{thm}
(Note that by the theorem's premise,  $\hat K$ is the matrix obtained from $K^*$ after zeroing out blocks corresponding to clusters of size at most $\ellsmall$.)
The proof is based on  \citet{chen2012sparseclustering} and is deferred to the supplemental material due to lack
of space. The main novelty in this work compared to previous work is the treatment of small clusters of size at most $\ellsmall$,
whereas in previous work only large clusters were treated, and the existence of small clusters did not allow recovery of
the big clusters.

\begin{defn}\label{defnpartial}
 An $n\times n$ matrix $K$ is a partial clustering  matrix if there exists a collection of pairwise disjoint sets 
 $U_1,\dots, U_r \subseteq [n]$ (the \emph{induced clusters})
 such that $K(i,j)=1$ if and only if $i,j \in U_s$ for some $s\in[r]$, otherwise $0$.
 If $K$ is a partial clustering matrix then $\sigmamin(K)$ is defined as $\min_{i=1}^r |U_i|$.
\end{defn}
The definition is depicted in Figure~\ref{optfig}.
Theorem~\ref{thm:main} tells us that by choosing $\kappa$ (and hence $c_1$, $c_2$) properly such that no cluster size
falls in the range $(\ellsmall, \ellbig)$, the  unique optimal solution $(\hat K, \hat B)$ to convex program (CP1) 
is such that $\hat K$ is a partial clustering induced by big ground truth clusters.

In order for this fact to be useful algorithmically, we also need a type of converse: there exists an event with high
probability (in the random process generating the input), such that for all values of $\kappa$, if an optimal solution 
to the corresponding (CP1) looks
like the solution $(\hat K, \hat B)$ defined in Theorem~\ref{thm:main}, 
then the blocks of $\hat K$ correspond to actual clusters.

\begin{figure}
\begin{tabular}{cc}
\scalebox{0.2}{\includegraphics{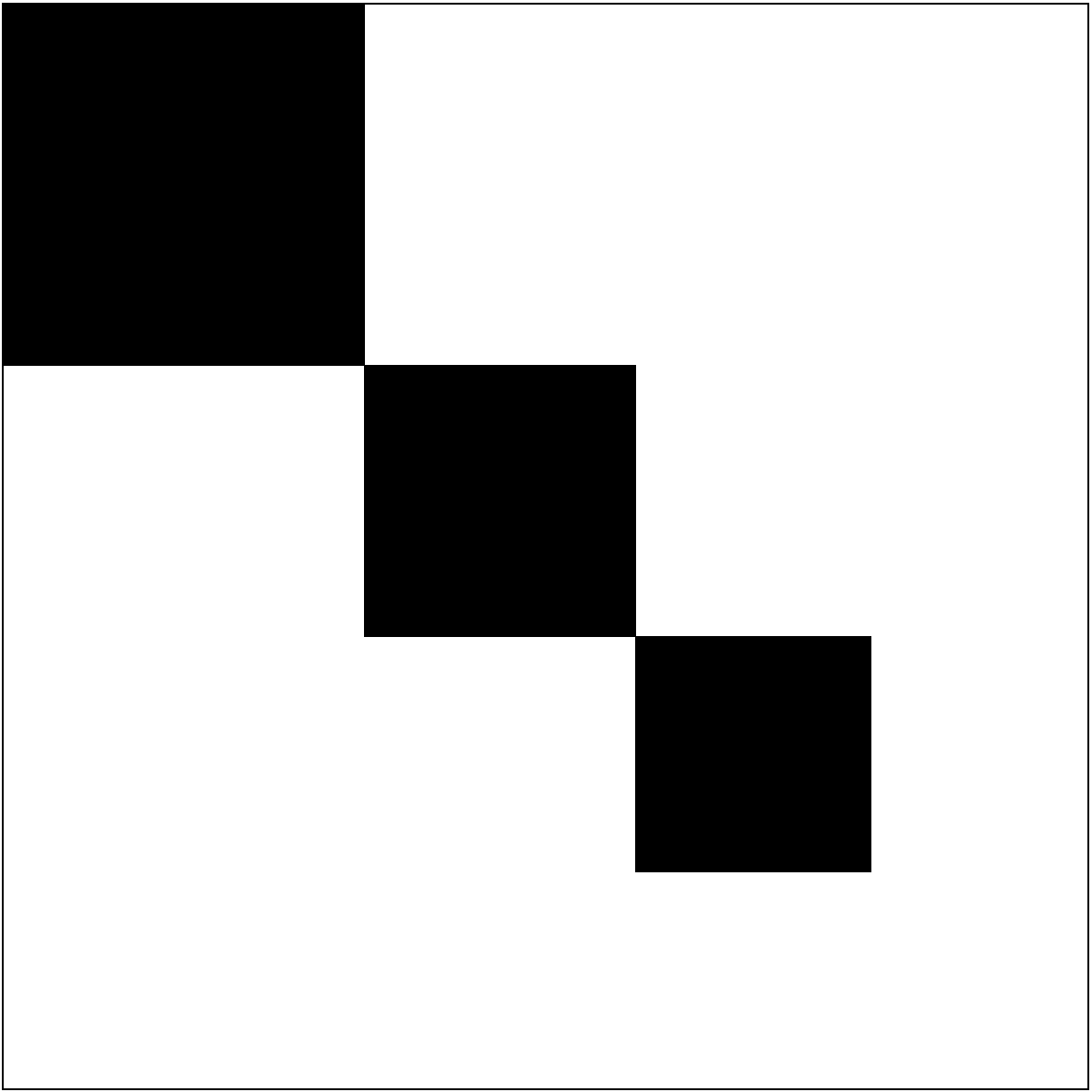}} 
&
\begin{minipage}{0.5\linewidth}
\vspace{-10ex}
 Black represents $1$, white represents $0$. $\sigmamin(K)$ is the
side length of the smallest black square.
\end{minipage}
\\
\end{tabular}
\caption{A partial clusering matrix $K$. }\label{optfig}
\end{figure}


\begin{thm}\label{soundness}
There exists constants $C_1,C_2>0$ such that
with probability at least $1-n^{-2}$, the following holds.
For all $\kappa \geq 1$ and $t\in[\frac 3 4 q + \frac 1 4 p, \frac 1 4 q + \frac 3 4 p]$, if $(K,B)$ is an optimal solution to (CP1) with 
$c_1,c_2$ as defined in Theorem~\ref{thm:main}, and additionally $K$ is a partial clustering induced by $U_1,\dots, U_r\subseteq V$, and also  
\begin{equation}\label{sigmaminbound}
\sigmamin(K) \geq \max  \left\{\frac{C_1 k\log n}{(p-q)^2}, \frac{ C_2 \kappa \sqrt{p(1-q)n\log n} }{p-q}\right\}\ ,\end{equation} then
$U_1,\dots, U_r$ are actual ground truth clusters,
 namely, there exists an injection $\phi : [r]\mapsto [k]$ such that $U_i = V_{\phi(i)}$ for all $i\in[r]$.
 \end{thm}
(Note: Our proof of Theorem~\ref{soundness} uses Hoeffding tail bounds for simplicity, which are tight for $p,q$ bounded away from $0$ and $1$.
Bernstein tail bounds can be used to strengthen the result for other classes of $p,q$.  We elaborate on this in Section~\ref{sec:partial}.)

The combination of Theorems~\ref{thm:main} and~\ref{soundness} implies that, 
as long as there exists 
a relatively small interval which is disjoint from the set of cluster sizes, and such that at least one cluster size is larger than this interval (and large enough),
we can recover at least one (large) cluster using (CP1).  This is made clear in the following.

\begin{cor}\label{getoneclust}
Assume we have a guarantee that there exists a number
$\alpha \geq b_4 \frac {\sqrt{p(1-q) n}}{p-q}$, such that no cluster size falls in the interval $(\alpha, \frac{b_3}{b_4}\alpha\log^2 n)$ and at least
one cluster size is of size at least $s := \max\{\frac{b_3}{b_4}\alpha\log^2 n, (C_1k\log n)/(p-q)^2, C_2\sqrt{p(1-q)n\log n}/(p-q)\} $.  Then with probability at least $1-n^{-2}$, we can recover 
 at least one cluster of size at least $s$ efficiently by  solving (CP1) with
$\kappa = \alpha/\left (b_4\frac{\sqrt{p(1-q)n}}{p-q}\right )$.
\end{cor}

Of course we do not know what $\alpha$ (and hence $\kappa$) is.  We could exhaustively search for a $\kappa\geq 1$ and
hope to recover at least one large cluster.  A more interesting question is, when is such a $\kappa$ guaranteed to exist?
Let $g=\frac {b_3}b{b_4}\log^2n$.  The number $g$ is the (multiplicative) gap size, equaling the ratio between $\ellbig$
and $\ellsmall$ (for any $\kappa$).
If the number of clusters $k$ is  a priori  bounded by some $k_0$, we both ensure that there is at least one cluster of size $n/k_0$,
and by the pigeonhole principle, that one of the intervals in the sequence
$(n/gk_0, n/k_0), (n/g^2k_0, n/gk_0),\dots,  (n/g^{k_0+1}k_0 ,n/g^{k_0}k_0)$.
is disjoint of cluster sizes.    If, in addition, the smallest interval in the sequence is not too small and $n/k_0$ is not
too small so that Corollary~\ref{getoneclust} holds, then we are guaranteed to recover at least one cluster using
Algorithm~\ref{algrecoverbigfull}.  We find this condition difficult to work with.  An elegant, useful version of the idea
is obtained if we assume  $p,q$ are some fixed constants.\footnote{In fact, we need only fix $(p-q)$, but we wish to keep
this exposition simple.} As the following lemma shows, it turns our that in this regime, $k_0$ can be assumed to be
almost logarithmic in $n$ to ensure recovery of at least one cluster.
\footnote{In comparison, \cite{AilonBE12} require $k_0$ to be constant for their guarantees, as do the  Correlation Clustering PTAS  \cite{GiotisGuruswami07}.}
  In what follows, notation such as $C(p,q), C_3(p,q),\dots$ denotes universal positive functions that depend on $p,q$
only.


\begin{lem}\label{findkappaalg}
There exists $C_3(p,q), C_4(p,q),C_5>0$ such that the following holds.
Assume that $n > C_4(p,q)$, and that we are guaranteed that $k \leq k_0$, where $k_0 = \frac{C_3(p,q)\log n}{\log\log n}$.
Then  with probability at least $1-n^{-2}$
Algorithm~\ref{algrecoverbigfull} will recover at least one cluster in at most $C_5 k_0$ iterations.
\end{lem}

The proof is deferred to the supplemental material section.
Lemma~\ref{findkappaalg} ensures that by trying at most a logarithmic number of values of $\kappa$, we can recover
at least one large cluster,  assuming the number of clusters is roughly logarithmic in $n$. 
The next proposition tells us that as long as this step recovers the clusters covering at most all but a vanishing fraction of elements, the step
can be repeated.
\def\nunrecovered{n''}
\def\kunrecovered{k''}
\begin{prop}\label{propinduction}
A pair of numbers $(n',k')$ is called good if $n'\leq n, k'\leq k$ and $k' \leq \frac {C_3(p,q)\log n'}{\log\log n}$. 
If $(n',k')$ is good, then $(\nunrecovered,\kunrecovered)$ is good for all $\nunrecovered, \kunrecovered$ satisfying
$n'\geq \nunrecovered \geq n'/(\log n)^{1/C_3(p,q)}$ and
$k'-1 \geq \kunrecovered \geq 1$.
\end{prop}
The proof is trivial. 
The proposition implies an inductive process in which at least one big (with respect to the current unrecovered size) cluster can be efficiently removed 
as long as the previous step recovered at most a  $(1-(\log n)^{-1/C_3(p,q)})$-fraction of its input. 
Combining, we proved the following:
\begin{thm}\label{algguarantee}
Assume $n,k$ satisfy the requirements of
Lemma~\ref{findkappaalg}.
Then with probability at least $1-2n^{-2}$
Algorithm~\ref{algrecoverfullobs} 
recovers clusters covering all but at most a $((\log n)^{-1/C_3(p-q)}) $ fraction of the input
in the full observation  case, without any restriction of the minimal cluster size.
Moreover, if we assume that $k$ is bounded by a \emph{constant} $k_0$, then the algorithm
will recover clusters covering all but a constant number of input elements.
\end{thm}

\def\recoverbig{{\operatorname{RecoverBigFullObs}}}
\begin{algorithm}[b!]
 \caption{$\recoverbig(V, A, p, q)$}
 \label{algrecoverbigfull}
 \begin{algorithmic}
   \STATE {\bf require:}  ground set $V$, $A\in \R^{V\times V}$,  probs $p,q$
   \STATE $n\leftarrow |V|$
   \STATE $t \leftarrow \frac 1 4 p + \frac 3 4 q$  (or anything in $[\frac 1 4 p + \frac 3 4 q, \frac 3 4 p + \frac 1 4 q]$)
   \STATE $\ellbig \leftarrow n$, $g\leftarrow \frac{b_3}{b_4}\log^2 n$
   \STATE // (If have prior bound $k_0$ on num clusters, 
    \STATE // $\ $ take $\ellbig\leftarrow n/k_0$)
   \WHILE {$\ellbig \geq  \max  \left\{\frac{C_1 k\log n}{(p-q)^2}, \frac{ C_2 \sqrt{p(1-q)n\log n} }{p-q}\right\}$}
      \STATE solve for $\kappa$ using (\ref{ells}), set $c_1,c_2$ as in (\ref{cees})
      \STATE $( \hat K, \hat  B) \leftarrow $ optimal solution to  (CP1)  with $c_1,c_2$
      \IF {$\hat K$ partial clustering matrix with $\sigmamin(\hat K)\geq \ellbig$}
          \STATE {\bf return} induced clusters $\{U_1,\dots, U_r\}$ of $\hat K$
      \ENDIF
      \STATE $\ellbig \leftarrow \ellbig / g$
   \ENDWHILE
   \STATE {\bf return} $\emptyset$
 \end{algorithmic}
\end{algorithm}

\def\recoverfull{\operatorname{RecoverFullObs}}
\begin{algorithm}
 \caption{$\recoverfull(V, A, p, q)$}
 \label{algrecoverfullobs}
 \begin{algorithmic}
 \STATE {\bf require:} ground set $V$, matrix $A\in \R^{V\times V}$, probs $p,q$
 \STATE $\{U_1,\dots, U_r\} \leftarrow \recoverbig(V,A,p,q)$
  \STATE $V' \leftarrow [n]\setminus (U_1\cup\cdots\cup U_r)$
  \IF {$r=0$} \STATE {\bf return} $\emptyset$ 
  \ELSE \STATE{\bf return} $\recoverfull(V', A[V'], p, q) \cup \{U_1,\dots, U_r\}$
  \ENDIF
\end{algorithmic}
\end{algorithm}

\subsection{Partial Observations}\label{sec:partial}

We now consider the case where the input matrix $A$ is not given to us in entirety, but rather that
we have oracle access to $A(i,j)$ for $(i,j)$ of our choice. Unobserved values are formally marked with $A(i,j) = ?$.

Consider a more particular setting in which the edge probabilities defining $A$ are $p'$ (for $i\sim j$) and $q'$ (for $i\not\sim j$),
and we observe $A(i,j)$ with probability $\rho$, for each $i,j$, independently.
More precisely: For $i\sim j$ we have $A(i,j)=1$ with probability $\rho p'$, $0$ with probability $\rho(1-p')$ and
$?$ with remaining probability.  For $i\not\sim j$ we have $A(i,j)=1$ with probability $\rho q'$, $0$ with probability
$\rho(1-q')$ and $?$ with remaining probability.
Clearly, by pretending that the values $?$ in $A$ are $0$, we emulate the full observation case with
$p=\rho p'$, $q=\rho q'$.  

Of particular interest is the case in which $p',q'$ are held fixed and $\rho$ tends to zero as $n$ grows.
In this regime, by varying $\rho$ and fixing $\kappa=1$, Theorem~\ref{thm:main} implies the following:
\begin{cor}
\label{cor:partial}There exist constants $b_{1}(p',q'),b_{3}(p',q'),b_{4}(p',q'),b_5(p',q')>0$ 
such that  for any sampling rate parameter $\rho$ the following holds with probability at least $1-n^{-3}$.
define 
\begin{eqnarray*}
\ellbig = b_{3}(p',q')\frac{\sqrt{n}}{\sqrt \rho}\log^{2}n \ \ \ \ \ \ \ 
\ellsmall= b_{4}(p',q')\frac{\sqrt{n}}{\sqrt \rho}\ .
\end{eqnarray*}
If for all $i\in [k]$,  either $n_i \geq \ellbig$ or $n_i \leq \ellsmall$ and if
 $(\hat K, \hat B)$ is  an optimal solution to (CP1), with 
\begin{eqnarray*}
c_{1}&=&\frac{b_{1}(p',q')}{ \sqrt{n\log n}}\sqrt{\frac{1-b_5(p',q')\rho}{b_5(p',q')\rho}} \\
c_{2}&=&\frac{b_{1}(p',q')}{ \sqrt{n\log n}}\sqrt{\frac{b_5(p',q')}{1-b_5(p',q')\rho}} \ ,
\end{eqnarray*}
then $(\hat K, \hat B) = (\Pbig K^*, A - \hat K)$, where $\Pbig$ is as defined in Theorem~\ref{thm:main}.
\end{cor}
(Note: We've abused notation by reusing previously defined global constants (e.g. $b_1$) with global  \emph{functions} of
$p',q'$ (e.g. $b_1(p',q')$).)
Notice now that the observation probability $\rho$ can be used as a knob for controlling the cluster sizes we are trying
to recover, instead of $\kappa$.
We would also like to obtain a version of Theorem~\ref{soundness}.  In particular, we would like to understand
its asymptotics as $\rho$ tends to $0$.

\begin{thm}\label{soundness2_}
There exist constants $C_1(p',q'),C_2(p',q')>0$  such that
for all observation rate parameters $\rho \leq 1$, the following holds
with probability at least $1-n^{-2}$.
If $(K,B)$ is an optimal solution to (CP1) with 
$c_1,c_2$ as defined in Theorem~\ref{cor:partial}, and additionally $K$ is a partial clustering induced 
by $U_1,\dots, U_r\subseteq V$, and also  
\begin{equation}\label{sigmaminbound2}
\sigmamin(K) \geq \max  \left\{\frac{C_1(p',q') k\log n}{\rho}, \frac{ C_2(p',q') \sqrt{ n\log n}} {\sqrt{\rho}}\right\}\ ,\end{equation} then
$U_1,\dots, U_r$ are actual ground truth clusters,
 namely, there exists an injection $\phi : [r]\mapsto [k]$ such that $U_i = V_{\phi(i)}$ for all $i\in[r]$.
 \end{thm}
The proof can be found in the supplemental material. 
Using the same reasoning as before, we derive the following:
\begin{thm}\label{findgappartial}
Let $g = b_3(p',q')/b_4(p',q')\log^2 n$ (with $b_3(p',q'), b_4(p',q')$ defined in Corollary~\ref{cor:partial}).
There exists a constant $C_4(p',q')$ such that the following holds.
Assume the number of clusters $k$ is bounded by some known number $k_0 \leq C_4(p',q') (\log n)/(\log\log n)$.
Let $\rho_0 = \frac{b_3(p',q')^2 k_0^2 \log^4 n}{n}$.
Then there exists $\rho$ in the set $\{\rho_0, \rho g, \dots, \rho g^{k_0}\}$
for which, if $A$ is obtained with observation rate $\rho$ (zeroing $?$'s), then with probability at least $1-n^{-2}$, 
any optimal solution $( K,  B)$ to (CP1) with $c_1,c_2$ from Corollary~\ref{cor:partial}
satisfies (\ref{sigmaminbound2}).
\end{thm}
(Note that the upper bound on $k_0$ ensures that $\rho g^{k_0}$ is a probability.)
The theorem is proven using a simple pigeonhole principle, noting that one of the intervals $(\ellsmall(\rho), \ellbig(\rho))$
must be disjoint from the set of cluster sizes, and there is at least one cluster of size at least $n/k_0$.
The theorem, together with Corollary~\ref{cor:partial} and Theorem~\ref{soundness2_} 
 ensures the following.  On one end of the spectrum, if $k_0$ is a constant (and $n$ is large enough), 
then with high probability we can recover at least one large cluster (of size at least  $n/k_0$) after querying no more
than 
\begin{equation}\label{boundkconst}O\left (n k_0^2 \left(\frac{b_3(p',q')}{b_4(p',q')}\log^2 n\right)^{2k_0}\log^4 n\right)  \end{equation} values of $A(i,j)$.
On the other end of the spectrum, if $k_0 \leq \delta (\log n)/(\log \log n)$ and $n$ is large enough (exponential in $1/\delta$),
then we can recover at least one large cluster after querying no more than $n^{1+O(\delta)}$ values of $A(i,j)$. (We
omit the details of the last fact from this version.)
This is summarized in the following:
\begin{thm}\label{thm:mainpartial}
Assume an upper bound $k_0$ on the number of clusters $k$.  As long as $n$ is larger than some function of $k_0, p',q'$,
Algorithm~\ref{algrecoverfullpart} will recover, with probability at least $1-n^{-1}$, at least one cluster of size at least $n/k_0$,
regardless of the size of other (small) clusters.
Moreover, if $k_0$ is a constant, then clusters covering all but a constant number of elements will be recovered with probability
at least $1-n^{-1}$, and the total number of observation queries is (\ref{boundkconst}), hence almost linear.
\end{thm}
Note that unlike previous results for this problem, the recovery guarantee does not impose any lower bound on the size
of the smallest cluster.
Also note that the underlying algorithm is an \emph{active learning} one, because more observations fall in
smaller clusters which survive deeper in the recursion of Algorithm~\ref{algrecoverfullpart}.
%
%



\def\recoverbigpart{{\operatorname{RecoverBigPartialObs}}}
\def\recoverpart{\operatorname{RecoverPartialObs}}

\begin{algorithm}[!b]
 \caption{$\recoverbigpart(V, k_0)$   (Assume $p',q'$ known, fixed)}
 \label{algrecoverbigpart}
 \begin{algorithmic}
   \STATE {\bf require:}  ground set $V$, oracle access to $A\in \R^{V\times V}$,  upper bound $k_0$ on number of clusters
   \STATE $n\leftarrow |V|$
    \STATE $\rho_0 \leftarrow \frac{b_3(p',q')^2 k_0^2 \log^4 n}{n}$
    \STATE  $g \leftarrow b_3(p',q')/b_4(p',q')\log^2 n$
   \FOR { $s \in \{0,\dots, k_0\} $}
      \STATE $\rho \leftarrow \rho_0 g^s$
     \STATE obtain matrix $A\in \{0,1,?\}^{V\times V}$ by sampling oracle at rate $\rho$, then zero $?$ values in $A$ 
     \STATE // (can reuse observations from prev. iterations)
      \STATE $c_1(p',q'),c_2(p',q')\leftarrow $ as in Corollary~\ref{cor:partial}
      \STATE $(K, B) \leftarrow $ an optimal solution to  (CP1) 
      \IF {$K$ partial clustering matrix satisfying (\ref{sigmaminbound2})}
          \STATE {\bf return} induced clusters $\{U_1,\dots, U_r\}$
      \ENDIF
   \ENDFOR
   \STATE {\bf return} $\emptyset$
 \end{algorithmic}
\end{algorithm}
\begin{algorithm}[!b]
 \caption{$\recoverpart(V, k_0)$     (Assume $p',q'$ known, fixed)}
 \label{algrecoverfullpart}
 \begin{algorithmic}
 \STATE {\bf require:} ground set $V$, oracle access to  $A\in \R^{V\times V}$, upper bound $k_0$ on number of clusters
 \STATE $\{U_1,\dots, U_r\} \leftarrow \recoverbig(V,k_0)$
  \STATE $V' \leftarrow [n]\setminus (U_1\cup\cdots\cup U_r)$
  \IF {$r=0$} \STATE {\bf return} $\emptyset$ 
  \ELSE \STATE{\bf return} $\recoverfull(V',k_0-r) \cup \{U_1,\dots, U_r\}$
  \ENDIF
\end{algorithmic}
\end{algorithm}



\section{Experiments}
\label{sub:expt}
We experimented with simplified versions of our algorithms.  Here we did not make an effort to compute the various constants defining the algorithms in this work, creating a difficulty in exact implementation. Instead, for Algorithm \ref{algrecoverbigfull}, we increase $ \kappa $ by a multiplicative factor of $ 1.1 $ in each iteration until a partial clustering matrix is found. Similarly, in Algorithm \ref{algrecoverbigpart}, $ \rho $ is increased by an additive factor of $ 0.025 $.  Still, it is obvious that our experiments support our theoretical findings.  A more practical ``user's guide'' for this method with actual constants is subject to future work.

In all experiment reports below, we use a variant of the Augmented Lagrangian Multiplier (ALM) method \cite{LinALM} to solve the semi-definite program (CP1).  Whenever we say that
``clusters $\{V_{i_1}, V_{i_2}, \dots\}$ were recovered'', we mean that a corresponding instantiation of (CP1)
resulted in an optimal solution $(K,B)$ for which $K$ was a partial clustering matrix induced by $\{V_{i_1}, V_{i_2}, \dots\}$.

\paragraph{Experiment 1 (Full Observation)}
Consider $ n=1100 $ nodes partitioned into $ 4 $ clusters $ V_1,\ldots, V_4 $, of sizes $800,200,80,20$, respectively. The  graph is generated according to the planted partition model with $ p=0.5 $, $ q=0. 2$, and we assume the full observation setting.
We apply a simplified version of Algorithm \ref{algrecoverfullobs}, which terminates in $ 4 $ steps. The recovered clusters at each step are detailed in Table \ref{resultstab}.    
%


\paragraph{Experiment 2 (Partial Observation - Fixed Sample Rate)}
We have $ n=1100 $ with clusters $V_1,\dots, V_4$ of sizes $800,200,50,50$. The observed graph is generated with $ p'=0.7 $, $ q'=0.1 $, and observation rate $ \rho = 0.3 $.   We repeatedly solved (CP1) with $c_1,c_2$ as in Corollary~\ref{cor:partial}.  At each iteration, at least one large cluster (compared to the input size at that iteration) was recovered
exactly and removed.
This terminated in exactly $3$ iterations.
Results are shown in Table \ref{resultstab}.
%
%

\paragraph{Experiment 3 (Partial Observation - Incremental Sampling Rate)}

We tried a simplified version of Algorithm~\ref{algrecoverfullpart}.
We have $ n=1100 $ with clusters $V_1,\dots, V_4$ of sizes $800,200,50,50$. The observed graph is generated with $ p'=0.7 $, $ q'=0.3 $, and an observation rate $ \rho$ which we now specify. We start with $ \rho=0 $ and increase it by $ 0.025 $ incrementally until we recover (and then remove) at least one cluster, then repeat.   The algorithm terminates in $ 3 $ steps. Results are shown in Table \ref{resultstab}.

\begin{table}[t]
\begin{center}
\begin{small}
\begin{sc}
{\bf Experiment 1:}
\begin{tabular}{ccccr}
\hline \hline
Step & $ \kappa $ & $\#$ nodes left & Clusters recovered \\
\hline
1    & 1 & 1100 & $V_1$ \\
2    & 1 & 300  & $V_2$\\
3    & 1 & 100  & $V_3$\\
4    & 1 & 20  & $V_4$ \\
\hline \hline
\end{tabular}
\end{sc}
\end{small}
\end{center}

\begin{center}
\begin{small}
\begin{sc}
{\bf Experiment 2:}
\begin{tabular}{ccccr}
\hline\hline
Step & $ \kappa $ & $\#$ nodes left & Clusters recovered \\
\hline
1    & 1 & 1100 & $V_1$ \\
2    & 1 & 300  & $V_2$\\
3    & 1 & 100  & $V_3$, $V_4$ \\
\hline\hline
\end{tabular}
\end{sc}
\end{small}
\end{center}

\begin{center}
\begin{small}
\begin{sc}
{\bf Experiment 3:}
\begin{tabular}{ccccr}
\hline\hline
Step & $ \rho $ & $\#$ nodes left & Clusters recovered \\
\hline
1    & 0.2 & 1100 & $V_1$ \\
2    & 0.4 & 300  & $V_2$\\
3    & 0.95 & 100  & $V_3$, $V_4$ \\
\hline\hline
\end{tabular}
\end{sc}
\end{small}
\end{center}

\begin{center}
\begin{small}
\begin{sc}
{\bf Experiment 3A:}
\begin{tabular}{ccccr}
\hline\hline
Step & $ \rho $ & $\#$ nodes left & Clusters recovered \\
\hline
1    & 0.15 & 4500 & $V_1$ \\
2    & 0.175 & 1300  & $V_2$\\
3    & 0.2 & 500  & $V_3$, $V_4$ \\
4  & 0.475 & 100  & $V_5$, $V_6$ \\
\hline\hline
\end{tabular}
\end{sc}
\end{small}
\end{center}
\caption{Experiment Results}\label{resultstab}
\end{table}
%

\paragraph{ Experiment 3A}
We repeat the experiment with a larger instance:  $ n=4500 $ with clusters $V_1,\dots,V_6$ of sizes $3200,800,200,200,50,50$, and $ p'=0.8,q'=0.2 $. Results are shown in Table \ref{resultstab}. Note that we recover the smallest clusters, whose size is below 
$ \sqrt{n} $.
%

\paragraph{Experiment 4 (Mid-Size Clusters)}

Our current theoretical results do not say anything about the mid-size clusters -- those with sizes between $ \ellsmall $ and $ \ellbig $.
 It is interesting to study the behavior of (CP1) in the presence of mid-size clusters. 
 We generated an instance with $ n=750 $, $ \{ |V_1|, |V_2|, |V_3|, |V_4| \} = \{500,150,70,30\} $, $ p=0.8, q=0.2 $, and $ \rho=0.12 $. 
 We then solved (CP1) with a fixed $ \kappa = 1 $. 
 The low-rank part $ \hat{K} $ of the solution is shown in Fig. \ref{midsizefig}. 
 The large cluster $ V_1 $ is completely recovered in $ \hat{K} $, while the small clusters $ V_3 $ and $ V_4 $ are entirely ignored. 
 The mid-size cluster $ V_2 $, however, exhibits a pattern we find difficult to characterize.
 This shows that the polylog gap in our theorems is a real phenomenon and not an artifact of our proof technique. 
Nevertheless, the large cluster appears clean, and might allow recovery using a  simple combinatorial procedure.  
 If this is true in general, it might not be necessary to search for a gap free of cluster sizes.
 Perhaps for any  $ \kappa $, (CP1) identifies all large clusters above $ \ellbig $ after a possible simple mid-size cleanup
procedure.  Understanding this phenomenon and its algorithmic implications is of much interest.

\begin{figure}[h!]
\center
\scalebox{0.40}{\includegraphics{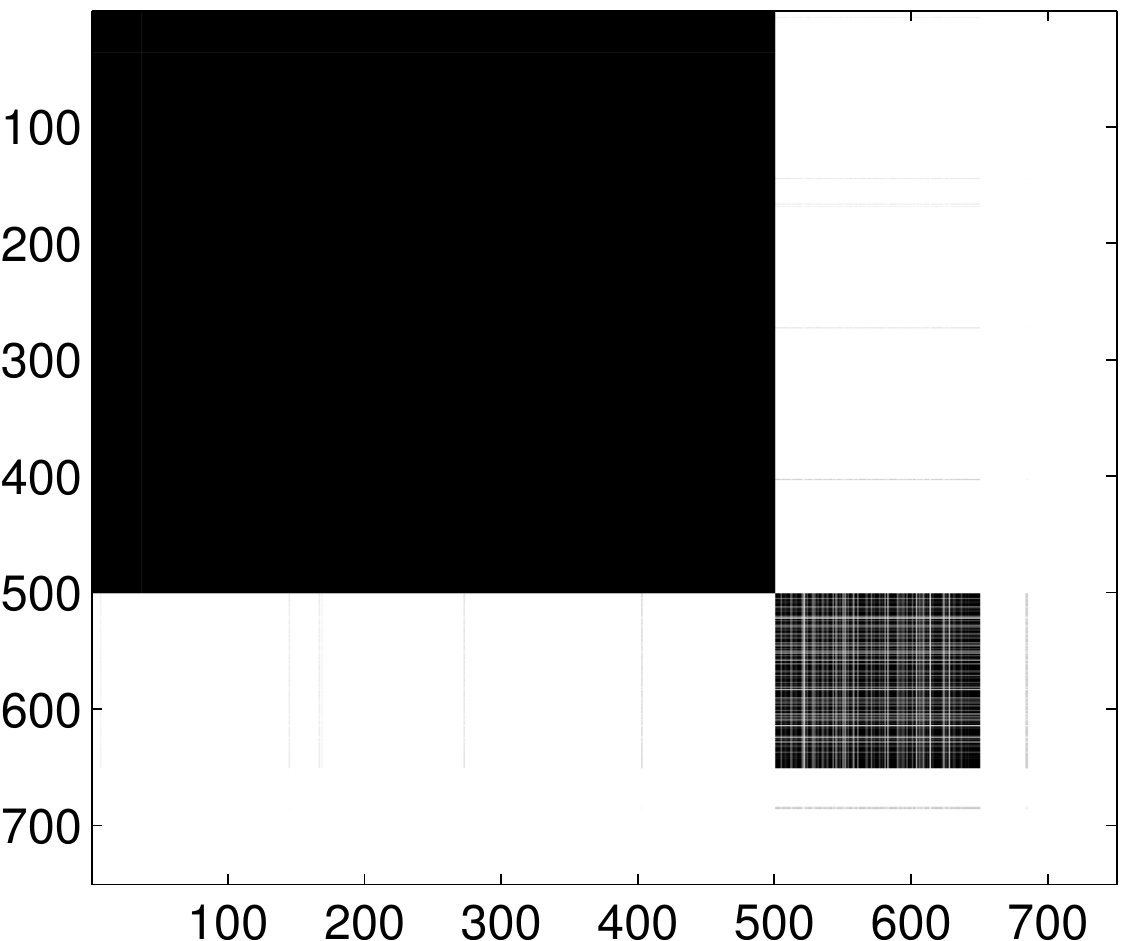}} 
\caption{The solution to (CP1) with mid-size clusters. }
\label{midsizefig}
\end{figure}



\section{Discussion}
An immediate future research is to better understand the ``mid-size crisis''. Our current results say nothing 
about clusters that are neither big nor small, falling in the interval $(\ellsmall, \ellbig)$.  
Our numerical experiments confirm that the mid-size phenomenon is real:
they are neither completely recovered 
nor entirely ignored by the optimal $\hat{K}$. The part of $\hat{K}$ restricted to these clusters does not seem to have an
obvious pattern.
Proving whether we can still efficiently recover large clusters in the presence of mid-size clusters is an interesting
open problem.

Our study was mainly theoretical, focusing on the planted partition model. As such, our experiments focused
on confirming  the theoretical findings with data generated  exactly according to the distribution we could provide
provable guarantees for.
It would be interesting to apply the presented methodology to real applications, 
particularly big data sets merged from web application and social networks. 

Another interesting direction is extending the ``peeling strategy'' to other high-dimensional learning problems. 
This requires understanding when such a strategy may work. One intuitive explanation of the small cluster barrier
encountered in previous work
is \emph{ambiguity}~--~when viewing the input at the ``big cluster resolution'', a small cluster  is both 
a low-rank matrix and a sparse matrix. Only when ``zooming in'' (after recovering big clusters), small clusters
patterns emerge.
 There are other formulations with similar property. For example, in \citet{XuCaramanisSanghaiv12-OP}, the authors propose to decompose a matrix into the sum of a low rank one and a column 
sparse one to solve an outlier-resistant PCA task. Notice that a column sparse matrix is also a low rank matrix. We hope 
the ``peeling strategy'' may also help with that problem.

\bibliographystyle{plainnat}
\bibliography{yudong,related,correlation_clustering}



\appendix

\section{Notation and Conventions}

We  use the following notation and conventions throughout the supplement.
For a real $n\times n$ matrix $M$, we use the unadorned norm $\|M\|$ to denote its spectral norm.
The notation $\|M\|_F$ refers to the Frobenius norm,  $\|M\|_1$ is  $\sum_{i,j}|M(i,j)|$ and $\|M\|_\infty$
is $\max_{ij} |M(i,j)|$.

We will also study operators on the space of matrices.    To distinguish them from the matrices studied in this work, we will simply
call these objects ``operators'', and  will denote them using a calligraphic font, e.g. $\P$.
The norm $\|\P\|$ of an operator is defined as 
$$ \|\P\| = \sup_{M: \|M\|_F=1} \|\P M\|_F\ ,$$
where the supremum is over matrices $M$.

For a fixed, real   $n\times n$ matrix $M$, we define the matrix linear subspace $T(M)$ as follows:

$$ T(M) := \{YM + MX:  X,Y\in \R^{n\times n}\}\ .$$
In words, this subspace is the set of matrices spanned by matrices each row of which is in the row space of $M$,  and matrices
each column of which is in the column space of $M$.

For any given subspace of matrices $S\subseteq \R^{n\times n}$, we let $\P_S$ denote the orthogonal projection onto $S$
with respect to the the inner product $\langle X,Y\rangle = \sum_{i,j=1}^n X(i,j)Y(i,j) = \trace X^t Y$.  This means that
for any matrix $M$,
$$ \P_S M = \operatorname{argmin}_{X\in S} \|M - X\|_F\ .$$

For a matrix $M$, we let $\Gamma(M)$ denote the set of matrices supported on a subset of the support of $M$.
Note that for any matrix $X$,
$$ (\P_{\Gamma(X)} M)(i,j) = \begin{cases} M(i,j) & X(i,j)\neq 0 \\ 0 & \mbox{otherwise} \end{cases}. $$

It is a well known fact that  $\P_{T(X)}$ is given as follows:
$$ \P_{T(X)} M = P_{C(X)} M + M P_{R(X)} - P_{C(X)} M P_{R(X)}\ ,$$
where $P_{C(X)}$ is projection (of a vector) onto the column space of $X$, and $P_{R(X)}$ is projection onto the row space of $X$.

For a subspace  $S\subseteq \R^{n\times n}$  we let $S^\perp$ denote the orthogonal subspace with respect to $\langle \cdot,\cdot\rangle$:
 $$ S^\perp = \{X\in \R^{n\times n}: \langle X, Y\rangle = 0 \ \forall Y\in S\}\ .$$  Slightly abusing notation, we will use the set complement operator $(\cdot)^c$ to formally define $\Gamma(M)^c$ to be $\Gamma(M)^\perp$ (by this we are stressing that the space $\Gamma(M)^\perp$ is given as  $\Gamma(M')$
where $M'$ is any matrix such that $M$ and $M'$ have complementary supports).
Note that
$ \P_{T(X)^\perp}M = M - \P_{T(X)}M =  (I-P_{C(X)})M(I-P_{R(X)})\ .$

For a matrix $M$, $\sign M$ is defined as the matrix satisfying:
$$ (\sign M)(i,j) = \begin{cases} 1 & M(i,j)> 0 \\ -1 & M(i,j) < 0 \\ 0 & \mbox{otherwise} \end{cases}\ .$$

\section{Proof of Theorem~\ref{thm:main}}

The proof is based on \cite{chen2012sparseclustering}.
We prove it for $\kappa=1$.  The adjustment for $\kappa>1$ is done using a padding argument, presented at the end of the proof.

Additional notation: 
\begin{enumerate}
\item 
We let $\Vsmall\subseteq V$ denote the set of of elements $i$ such that $n_{\langle i \rangle} \leq \ellsmall$.
(We remind the reader that $n_{\langle i\rangle} = |V_{\langle i\rangle}|$.)
\item We remind the reader that the projection $\Pbig$ is defined as follows:
$$ (\Pbig M)(i,j) = \begin{cases} M(i,j) & \max\{n_{\langle i\rangle}, n_{\langle j\rangle}\} \geq \ellbig \\ 0 & \mbox{otherwise} \end{cases}\ .
$$ 
\item 
The projection $\Psmall$ is defined as follows:
$$ (\Psmall  M)(i,j) = \begin{cases} M(i,j) & \max\{n_{\langle i\rangle}, n_{\langle j\rangle}\} \leq  \ellsmall \\ 0 & \mbox{otherwise} \end{cases}\ .
$$ 
In words, $\Psmall$ projects onto the set of matrices supported on $\Vsmall\times \Vsmall$.
Note that by the theorem assumption, $\Pbig + \Psmall = \idop$ (equivalently, $\Pbig$ projects onto the set of matrices
supported on $(V\times V)\setminus (\Vsmall \times \Vsmall)$).
\item Define the set 
\[
\DD=\left\{ \Delta\in\mathbb{R}^{n\times n}\vert\Delta_{ij}\le0,\forall i\sim j,(i,j)\notin\Vsmall\times\Vsmall;0\le\Delta_{ij},\forall i\not\sim j,(i,j)\notin\Vsmall\times\Vsmall\right\} ,
\]
which contains all feasible deviation from $\hat{K}$. 
\item For simplicity we write $T:=T(\hat{K})$ and $\Gamma:=\Gamma(\hat{B}), \Gamma^c := \Gamma(\hat B)^c = \Gamma^\perp$.
\end{enumerate}
We will make use of the following:
\begin{enumerate}
\item $\sign(\hat{B})=\hat{B}$.
\item $\idop=\PP_{\Gamma}+\PP_{\Gamma^{c}}=\PP_{\Gamma(A)}+\PP_{\Gamma(A)^{c}}$.
\item $\Pbig,\Psmall,\PP_{\Gamma},\PP_{\Gamma^{c}},\PP_{\Gamma(A)},$ and
$\PP_{\Gamma(A)^{c}}$ commute with each other.
\end{enumerate}

\subsection{Approximate Dual Certificate Condition}
\begin{prop}
\label{prop:opt_cond}$(\hat{K}$, $\hat{B}$) is the unique optimal
solution to (CP) if there exists a matrix $Q\in\mbox{\ensuremath{\mathbb{R}}}^{n\times n}$
and a positive number $\epsilon$ satisfying:
\begin{enumerate}
\item $\left\Vert Q\right\Vert <1$
\item $\left\Vert \PP_{T}(Q)\right\Vert _{\infty}\le\frac{\epsilon}{2}\min\left\{ c_{1},c_{2}\right\} $
\item $\forall\Delta\in\DD$:

\begin{enumerate}
\item $\left\langle UU^{\top}+Q,\PP_{\Gamma(A)}\PP_{\Gamma}\Pbig\Delta\right\rangle =(1+\epsilon)c_{1}\left\Vert \PP_{\Gamma(A)}P_{\Gamma}\Pbig\Delta\right\Vert _{1}$
\item $\left\langle UU^{\top}+Q,\PP_{\Gamma(A)^{c}}\PP_{\Gamma}\Pbig\Delta\right\rangle =(1+\epsilon)c_{2}\left\Vert \PP_{\Gamma(A)^{c}}P_{\Gamma}\Pbig\Delta\right\Vert _{1}$
\end{enumerate}
\item $\forall\Delta\in\DD$:

\begin{enumerate}
\item $\left\langle UU^{\top}+Q,\PP_{\Gamma(A)}\PP_{\Gamma^{c}}\PP_{\sharp}\Delta\right\rangle \ge-(1-\epsilon)c_{1}\left\Vert \PP_{\Gamma(A)}P_{\Gamma^{c}}\Pbig\Delta\right\Vert _{1}$
\item $\left\langle UU^{\top}+Q,\PP_{\Gamma(A)^{c}}\PP_{\Gamma^{c}}\PP_{\sharp}\Delta\right\rangle \ge-(1-\epsilon)c_{2}\left\Vert \PP_{\Gamma(A)^{c}}P_{\Gamma^{c}}\Pbig\Delta\right\Vert _{1}$
\end{enumerate}
\item $\PP_{\Gamma}\Psmall(UU^{\top}+Q)=c_{1}\Psmall\hat{B}$
\item $\left\Vert \PP_{\Gamma^{c}}\Psmall(UU^{\top}+Q)\right\Vert _{\infty}\le c_{2}$
\end{enumerate}
\end{prop}
\begin{proof}
Consider any feasible solution to (CP1) $(\hat{K}+\Delta,\hat{B}-\Delta)$;
we know $\Delta\in\DD$ due to the inequality constraints in (CP1).
We will show that this solution will have strictly higher objective
value than $\left\langle \hat{K},\hat{B}\right\rangle $ if $\Delta\neq0$.

For this $\Delta$, let $G_{\Delta}$ be a matrix in $T^{\bot}\cap\mbox{Range}(\Psmall)$
satisfying $\left\Vert G\right\Vert =1$ and $\left\langle G_{\Delta},\Delta\right\rangle =\left\Vert \PP_{T^{\bot}}\Psmall\Delta\right\Vert _{*}$
; such a matrix always exists because $\mbox{Range}\Psmall\subseteq T^{\bot}$.
Suppose $\left\Vert Q\right\Vert =b$. Clearly, $\PP_{T^{\bot}}Q+(1-b)G\in T^{\bot}$
and, due to desideratum 1, we have $\left\Vert \PP_{T^{\bot}}Q+(1-b)G_{\Delta}\right\Vert \le\left\Vert Q\right\Vert +(1-b)\left\Vert G_{\Delta}\right\Vert =b+(1-b)=1$.
Therefore, $UU^{\top}+\PP_{T^{\bot}}Q+(1-b)G_{\Delta}$ is a subgradient
of $f(K)=\left\Vert K\right\Vert _{*}$ at $K=\hat{K}$. On the other
hand, define the matrix $F_{\Delta}=-\PP_{\Gamma^{c}}\mbox{sgn}(\Delta)$.
We have $F_{\Delta}\in\Gamma^{c}$ and $\left\Vert F_{\Delta}\right\Vert _{\infty}\le1$.
Therefore, $\PP_{\Gamma(A)}(\hat{B}+F_{\Delta})$ is a subgradient
of $g_{1}(B)=\left\Vert \PP_{\Gamma(A)}B\right\Vert _{1}$ at $B=\hat{B}$,
and $\PP_{\Gamma(A)^{c}}(\hat{B}+F_{\Delta})$ is a subgradient of
$g_{2}(B)=\left\Vert \PP_{\Gamma(A)^{c}}B\right\Vert _{1}$ at $B=\hat{B}$.
Using these three subgradients, the difference in the objective value
can be bounded as follows:
\begin{eqnarray*}
 &  & d(\Delta)\\
 & \triangleq & \left\Vert \hat{K}+\Delta\right\Vert _{*}+c_{1}\left\Vert \PP_{\Gamma(A)}(\hat{B}-\Delta)\right\Vert _{1}+c_{2}\left\Vert \PP_{\Gamma(A)^{c}}(\hat{B}-\Delta)\right\Vert _{1}-\left\Vert \hat{K}\right\Vert _{*}-c_{1}\left\Vert \PP_{\Gamma(A)}\hat{B}\right\Vert _{1}-c_{2}\left\Vert \PP_{\Gamma(A)^{c}}\hat{B}\right\Vert _{1}\\
 & \ge & \left\langle UU^{\top}+\PP_{T^{\bot}}Q+(1-b)G_{\Delta},\Delta\right\rangle +c_{1}\left\langle \PP_{\Gamma(A)}(\hat{B}+F_{\Delta}),-\Delta\right\rangle +c_{2}\left\langle \PP_{\Gamma(A)^{c}}(\hat{B}+F_{\Delta}),-\Delta\right\rangle \\
 & = & (1-b)\left\Vert \PP_{T^{\bot}}\Psmall\Delta\right\Vert _{*}+\left\langle UU^{\top}+\PP_{T^{\bot}}Q,\Delta\right\rangle +c_{1}\left\langle \PP_{\Gamma(A)}\hat{B},-\Delta\right\rangle +c_{2}\left\langle \PP_{\Gamma(A)^{c}}\hat{B},-\Delta\right\rangle \\
 &  & +c_{1}\left\langle \PP_{\Gamma(A)}F_{\Delta},-\Delta\right\rangle +c_{2}\left\langle \PP_{\Gamma(A)^{c}}F_{\Delta},-\Delta\right\rangle \\
 & = & (1-b)\left\Vert \PP_{T^{\bot}}\Psmall\Delta\right\Vert _{*}+\left\langle UU^{\top}+\PP_{T^{\bot}}Q,\Delta\right\rangle +c_{1}\left\langle \Psmall\PP_{\Gamma(A)}\hat{B},-\Delta\right\rangle +c_{2}\left\langle \Psmall\PP_{\Gamma(A)^{c}}\hat{B},-\Delta\right\rangle \\
 &  & +c_{1}\left\langle \Pbig\PP_{\Gamma(A)}\hat{B},-\Delta\right\rangle +c_{2}\left\langle \Pbig\PP_{\Gamma(A)^{c}}\hat{B},-\Delta\right\rangle +c_{1}\left\langle \PP_{\Gamma(A)}F_{\Delta},-\Delta\right\rangle +c_{2}\left\langle \PP_{\Gamma(A)^{c}}F_{\Delta},-\Delta\right\rangle .
\end{eqnarray*}
The last six terms of the last RHS satisfy:
\begin{enumerate}
\item $c_{1}\left\langle \Psmall\PP_{\Gamma(A)}\hat{B},-\Delta\right\rangle +c_{2}\left\langle \Psmall\PP_{\Gamma(A)^{c}}\hat{B},-\Delta\right\rangle =c_{1}\left\langle \Psmall\hat{B},-\Delta\right\rangle $,
because $\Psmall\hat{B}\in\Gamma(A)$.
\item $\left\langle \Pbig\PP_{\Gamma(A)}\hat{B},-\Delta\right\rangle \ge-\left\Vert \Pbig\PP_{\Gamma(A)}\PP_{\Gamma}\Delta\right\Vert _{1}$
and $\left\langle \Pbig\PP_{\Gamma(A)^{c}}\hat{B},\Delta\right\rangle \ge-\left\Vert \Pbig\PP_{\Gamma(A)^{c}}\PP_{\Gamma}\Delta\right\Vert _{1}$,
because $\hat{B}\in\Gamma$ and$\left\Vert \hat{B}\right\Vert _{\infty}\le1$.
\item $\left\langle \PP_{\Gamma(A)}F_{\Delta},-\Delta\right\rangle =\left\Vert \PP_{\Gamma(A)}\PP_{\Gamma^{c}}\Delta\right\Vert _{1}$
and $\left\langle \PP_{\Gamma(A)^{c}}F_{\Delta},-\Delta\right\rangle =\left\Vert \PP_{\Gamma(A)^{c}}\PP_{\Gamma^{c}}\Delta\right\Vert _{1}$,
due to the definition of $F$. 
\end{enumerate}
It follows that
\begin{eqnarray}
d(\Delta) & \ge & (1-b)\left\Vert \PP_{T^{\bot}}\Psmall\Delta\right\Vert _{*}+\left\langle UU^{\top}+\PP_{T^{\bot}}Q,\Delta\right\rangle +c_{1}\left\langle \Psmall\hat{B},-\Delta\right\rangle -c_{1}\left\Vert \Pbig\PP_{\Gamma(A)}\PP_{\Gamma}\Delta\right\Vert _{1}\nonumber \\
 &  & -c_{2}\left\Vert \Pbig\PP_{\Gamma(A)^{c}}\PP_{\Gamma}\Delta\right\Vert _{1}+c_{1}\left\Vert \PP_{\Gamma(A)}\PP_{\Gamma^{c}}\Delta\right\Vert _{1}+c_{2}\left\Vert \PP_{\Gamma(A)^{c}}\PP_{\Gamma^{c}}\Delta\right\Vert _{1}.\label{eq:1}
\end{eqnarray}
Consider the second term in the last RHS, which equals $\left\langle UU^{\top}+\PP_{T^{\bot}}Q,\Delta\right\rangle =\left\langle UU^{\top}+Q,\Pbig\Delta\right\rangle +\left\langle UU^{\top}+Q,\Psmall\Delta\right\rangle -\left\langle \PP_{T}Q,\Delta\right\rangle $.
We bound these three separately.

First term:
\begin{eqnarray*}
 &  & \left\langle UU^{\top}+Q,\Pbig\Delta\right\rangle \\
 & = & \left\langle UU^{\top}+Q,\left(\PP_{\Gamma(A)}\PP_{\Gamma}\Pbig+\PP_{\Gamma(A)^{c}}\PP_{\Gamma}\Pbig+\PP_{\Gamma(A)}\PP_{\Gamma^{c}}\Pbig+\PP_{\Gamma(A)^{c}}\PP_{\Gamma^{c}}\Pbig\right)\Delta\right\rangle \\
 & \ge & (1+\epsilon)c_{1}\left\Vert \PP_{\Gamma(A)}P_{\Gamma}\Pbig\Delta\right\Vert _{1}+(1+\epsilon)c_{2}\left\Vert \PP_{\Gamma(A)^{c}}P_{\Gamma}\Pbig\Delta\right\Vert _{1}-(1-\epsilon)c_{1}\left\Vert \PP_{\Gamma(A)}P_{\Gamma^{c}}\Pbig\Delta\right\Vert _{1}\\
 &  & -(1-\epsilon)c_{2}\left\Vert \PP_{\Gamma(A)^{c}}P_{\Gamma^{c}}\Pbig\Delta\right\Vert _{1}\\
 &  & \mbox{(Using properties 3 and 4)}
\end{eqnarray*}

Second term:
\begin{eqnarray*}
 &  & \left\langle UU^{\top}+Q,\Psmall\Delta\right\rangle \\
 & = & \left\langle \PP_{\Gamma}\Psmall(UU^{\top}+Q),\Delta\right\rangle +\left\langle \PP_{\Gamma^{c}}\Psmall(UU^{\top}+Q),\Delta\right\rangle \\
 & \ge & c_{1}\left\langle \Psmall\hat{B},\Delta\right\rangle -c_{2}\left\Vert \PP_{\Gamma^{c}}\Psmall\Delta\right\Vert _{1}\mbox{ (using properties 5 and 6)}\\
 & = & c_{1}\left\langle \Psmall\hat{B},\Delta\right\rangle -c_{2}\left\Vert \PP_{\Gamma(A)^{c}}\PP_{\Gamma^{c}}\Psmall\Delta\right\Vert _{1}\mbox{ (because \ensuremath{\PP_{\Gamma(A)^{c}}\PP_{\Gamma^{c}}\Psmall=\PP_{\Gamma^{c}}\Psmall})}
\end{eqnarray*}

Third term: Due to the block diagonal structure of the elements of
$T$, we have $\PP_{T}=\Pbig\PP_{T}$
\begin{eqnarray*}
 &  & \left\langle -\PP_{T}Q,\Delta\right\rangle \\
 & = & -\left\langle \PP_{T}Q,\Pbig\Delta\right\rangle \\
 & \ge & -\left\Vert \PP_{T}Q\right\Vert _{\infty}\left\Vert \Pbig\Delta\right\Vert _{1}\\
 & \ge & -\frac{\epsilon}{2}\min\left\{ c_{1},c_{2}\right\} \left\Vert \Pbig\Delta\right\Vert _{1}.
\end{eqnarray*}

Combining the above three bounds with Eq. (\ref{eq:1}), we obtain
\begin{eqnarray*}
 &  & d(\Delta)\\
 & \ge & (1-b)\left\Vert \PP_{T^{\bot}}\Psmall\Delta\right\Vert _{*}+\epsilon c_{1}\left\Vert \Pbig\PP_{\Gamma(A)}\PP_{\Gamma}\Delta\right\Vert _{1}+\epsilon c_{2}\left\Vert \Pbig\PP_{\Gamma(A)^{c}}\PP_{\Gamma}\Delta\right\Vert _{1}+\epsilon c_{1}\left\Vert \PP_{\Gamma(A)}\PP_{\Gamma^{c}}\Pbig\Delta\right\Vert _{1}\\
 &  & +\epsilon c_{2}\left\Vert \PP_{\Gamma(A)^{c}}\PP_{\Gamma^{c}}\Pbig\Delta\right\Vert _{1}+c_{1}\left\Vert \PP_{\Gamma(A)}\PP_{\Gamma^{c}}\Psmall\Delta\right\Vert _{1}-\frac{\epsilon}{2}\min\left\{ c_{1},c_{2}\right\} \left\Vert \Pbig\Delta\right\Vert _{1}\\
 & = & (1-b)\left\Vert \PP_{T^{\bot}}\Psmall\Delta\right\Vert _{*}+\epsilon c_{1}\left\Vert \Pbig\PP_{\Gamma(A)}\Delta\right\Vert _{1}+\epsilon c_{2}\left\Vert \Pbig\PP_{\Gamma(A)^{c}}\Delta\right\Vert _{1}-\frac{\epsilon}{2}\min\left\{ c_{1},c_{2}\right\} \left\Vert \Pbig\Delta\right\Vert _{1}\\
 &  & \mbox{ (note that \ensuremath{\PP_{\Gamma(A)}\PP_{\Gamma^{c}}\Psmall\Delta}=0)}\\
 & \ge & (1-b)\left\Vert \Psmall\Delta\right\Vert _{*}+\frac{\epsilon}{2}\min\left\{ c_{1},c_{2}\right\} \left\Vert \Pbig\Delta\right\Vert _{1},
\end{eqnarray*}
which is strictly greater than zero for $\Delta\neq0$.
\end{proof}

\subsection{Constructing $Q$.}

We construct a matrix $Q$ with the properties required by Proposition
\ref{prop:opt_cond}. Suppose we take $\epsilon=\frac{2\log^{2}n}{\ellbig}\sqrt{\frac{n}{t(1-t)}}$
and use the weights $c_{1}$ and $c_{2}$ given in Theorem \ref{thm:main}.
We specify $\Pbig Q$ and $\Psmall Q$ separately.

$\Pbig Q$ is given by $\Pbig Q=\Pbig Q_{1}+\Pbig Q_{2}+\Pbig Q_{3}$,
where for $(i,j)\notin\Vsmall\times\Vsmall$, 
\begin{eqnarray*}
\Pbig Q_{1}(i,j) & = & \begin{cases}
-\frac{1}{n_{c(i)}} & i\sim j,(i,j)\in\Gamma\\
\frac{1}{n_{c(i)}}\cdot\frac{1-p_{ij}}{p_{ij}} & i\sim j,(i,j)\in\Gamma^{c}\\
0 & i\not\sim j
\end{cases}\\
\Pbig Q_{2}(i,j) & = & \begin{cases}
-(1+\epsilon)c_{2} & i\sim j,(i,j)\in\Gamma\\
(1+\epsilon)c_{2}\frac{1-p_{ij}}{p_{ij}} & i\sim j,(i,j)\in\Gamma^{c}\\
0 & i\not\sim j
\end{cases}\\
\Pbig Q_{3}(i,j) & = & \begin{cases}
(1+\epsilon)c_{1} & i\not\sim j,(i,j)\in\Gamma\\
-(1+\epsilon)c_{1}\frac{q_{ij}}{1-q_{ij}} & i\not\sim j,(i,j)\in\Gamma^{c}\\
0 & i\sim,j
\end{cases}
\end{eqnarray*}
Note that these matrices have zero-mean entries.

$\Psmall Q$ as follows. For $(i,j)\in\Vsmall\times\Vsmall$,
\[
\Psmall Q=\begin{cases}
c_{1} & i\sim j,(i,j)\in\Gamma(A)\\
-c_{2} & i\sim j,(i,j)\in\Gamma(A)^{c}\\
c_{1} & i\not\sim j,(i,j)\in\Gamma(A)\\
c_{2}W(i,j) & i\not\sim j,(i,j)\in\Gamma(A)^{c}
\end{cases},
\]
where 
\[
W(i,j)=\begin{cases}
+1 & \mbox{with probability \ensuremath{\frac{t-q}{2t(1-q)}}}\\
-1 & \mbox{with remaining probability}
\end{cases}.
\]

\subsection{Validating $Q$}

Under the choice of $ t $ in Theorem \ref{thm:main}, we have $ \frac{1}{4}p \le t\le p $ and $\frac{1}{4} (1-q) \le 1-t\le 1-q $. 
Also under the second assumption of the theorem and $ p-q\le p(1-q) $, we have $ p(1-q) \gtrsim \frac{n \log^4 n}{\ellbig^2} \ge  \frac{ \log^4 n}{\ellbig}$.
We will make use of these facts frequently in the proof. 

It is easy to check that $\epsilon:=\frac{2\log^{2}n}{\ellbig}\sqrt{\frac{n}{t(1-t)}}<\frac{1}{2}$
under the assumption of Theorem \ref{thm:main}.

\textbf{Property 1): }

Note that $\left\Vert Q\right\Vert \le\left\Vert \Pbig Q_{\sim}\right\Vert +\left\Vert \Pbig Q_{\not\sim}\right\Vert +\left\Vert \Psmall Q_{\sim}\right\Vert +\left\Vert \Psmall Q_{\not\sim}\right\Vert $.
We show that all four terms are upper-bounded by $\frac{1}{4}$. 

(a) $\Psmall Q_{\sim}$ is a block diagonal matrix with each block having
size at most $\mbox{\ensuremath{\ellsmall}}$. Moreover, $\Psmall Q_{\sim}$ is the sum of
a deterministic matrix $ Q_{\sim,d} $ with all non-zero entries equal to $\frac{b_1}{\sqrt{n\log n}}\frac{p-t}{\sqrt{t(1-t)}}$
and a random matrix $ Q_{\sim,r} $ whose entries are i.i.d., bounded almost surely by $\max\{c_1,c_2\}  $  and have zero mean 
with variance $\frac{b_1^2}{n\log n}\cdot\frac{p(1-p)}{t(1-t)}$.
Therefore, we have $ \Vert Q_{\sim} \Vert \le \Vert Q_{\sim,r} \Vert + \Vert Q_{\sim,d}\Vert$, where w.h.p.
\begin{eqnarray*}
\left\Vert \Psmall Q_{\sim,d}\right\Vert & \le  &\ellsmall\frac{b_1}{\sqrt{n\log n}}\frac{p-t}{\sqrt{t(1-t)}}, \\
\left\Vert \Psmall Q_{\sim,r}\right\Vert & \le & 6\max\left\{\sqrt{\ellsmall \log n}\frac{b_1}{\sqrt{n\log n}}\sqrt{\frac{p(1-p)}{t(1-t)}}
 + \max\{c_1,c_2\}\log^2 n \right\};
\end{eqnarray*}
here in the second inequality we use Lemma \ref{lem:rand_matrix}. We conclude that $ \left\Vert \Psmall Q_{\sim}\right\Vert $  is bounded by $\frac{1}{4}$ 
as long as $\ellsmall \le \frac{\sqrt{t(1-t)n\log n}}{8b_1(p-t)}$ and $ \max\{c_1,c_2\} \le \frac{1}{48\log^2 n} $,
which holds under the assumption of Theorem \ref{thm:main}. 

(b) $\Psmall Q_{\not\sim}$ is a random matrix supported
on $\Vsmall\times\Vsmall$, whose entries are i.i.d., zero mean, bounded almost surely by $ \max\{c_1,c_2\} $, 
and have variance $\frac{b_1^2}{n\log n}\cdot\frac{t^{2}+q-2tq}{(1-t)t}$.
It follows from Lemma \ref{lem:rand_matrix} that
\[
\left\Vert \Psmall Q_{\not\sim}\right\Vert \le  6 \max\left\{\sqrt{\nsmall}\cdot\frac{b_1}{\sqrt{n\log n}}\sqrt{\frac{t^{2}+q-2tq}{(1-t)t}},
\max\{c_1,c_2\}\log^2 n \right\} 
\le\frac{1}{4}
\]
because $\nsmall\le n$ and $ \max\{c_1,c_2\} \le \frac{1}{48\log^2 n} $,
which holds under the assumption of the theorem. 

(c) Note that $\Pbig Q_{\sim}=\Pbig Q_{1}+\Pbig Q_{2}$. By construction
these two matrices are both block-diagonal, have i.i.d zero-mean entries which are bounded almost surely 
by $ B_\sim := \max\left\{\frac{1}{\ellbig p}, \frac{2c_2}{p} \right\} $
and have variance
bounded by $\sigma_{\sim}^{2}:=\max\left\{ \frac{1-p}{p\ellbig^{2}},\frac{4(1-p)}{p}c_{2}^{2}\right\} $.
Lemma \ref{lem:rand_matrix} gives $\left\Vert \Pbig Q_{\sim}\right\Vert \le 
6\max\left\{\sqrt{\ellbig}\cdot\sigma_{\sim}, B_\sim \log^2n  \right\}\le\frac{1}{4}$
under the assumption of Theorem \ref{thm:main}.

(d) Note that $\Pbig Q_{\not\sim}=\Pbig Q_{3}$ is a random
matrix with i.i.d. zero-mean entries which are bounded almost surely by $ B_{\not\sim}:= \frac{2c_1}{1-q} $
and have variance bounded by $\sigma_{\not\sim}^{2}:=\frac{4q}{1-q}c_{1}^{2}$.
Lemma \ref{lem:rand_matrix} gives $\left\Vert \Pbig Q_{\not\sim}\right\Vert \le
6\max\left\{\sqrt{n}\cdot\sigma_{\not\sim}, B_{\not\sim}\log^2n\right\}\le\frac{1}{4}$.

\textbf{Property 2):}

Due to the structure of $T$, we have 
\begin{eqnarray*}
\left\Vert \PP_{T}Q\right\Vert _{\infty} & = & \left\Vert \PP_{T}\Pbig Q\right\Vert _{\infty}=\left\Vert UU^{\top}(\Pbig Q)+(\Pbig Q)UU^{\top}+UU^{\top}(\Pbig Q)UU^{\top}\right\Vert _{\infty}\\
 & \le & 3\left\Vert UU^{\top}\Pbig Q\right\Vert _{\infty}\le3\sum_{m=1}^{3}\left\Vert UU^{\top}\Pbig Q_{m}\right\Vert _{\infty}.
\end{eqnarray*}
Now observe that $(UU^{\top}\Pbig Q_{m})(i,j)=\sum_{l\in V_{c(i)}}\frac{1}{n_{c(i)}}\Pbig Q_{m}(l,j)$
is the sum of i.i.d. zero-mean random variables with bounded magnitude and variance.
Using Lemma \ref{lem:subgaussian}, we obtain that for $i\in\Vbig$,
\begin{eqnarray*}
\left|(UU^{\top}\Pbig Q_{1})(i,j)\right| & \lesssim & \frac{1}{n_{c(i)}}\left(\sqrt{\frac{1-p}{p\ellbig^{2}}}\cdot\sqrt{n_{c(i)}\log n}
+ \frac{\log n}{\ellbig p}\right)\\
 & \le & \frac{1}{\ellbig}\sqrt{\frac{\log n}{p\ellbig}}\le\frac{\log n}{24^{2}\ellbig}\sqrt{\frac{t}{p}}.
\end{eqnarray*}
where in the last inequality we use $t \ge \frac{p}{4}\gtrsim\frac{\log n}{\ellbig}$. For $i\in\Vsmall$,
clearly $(UU^{\top}\Pbig Q_{1})(i,j)=0$. By union bound we conclude
that $\left\Vert UU^{\top}\Pbig Q_{1}\right\Vert _{\infty}\le\frac{\log n}{24^{2}\ellbig}\sqrt{\frac{t}{p}}.$
Similarly, we can bound $\left\Vert UU^{\top}\Pbig Q_{2}\right\Vert _{\infty}$
and $\left\Vert UU^{\top}\Pbig Q_{3}\right\Vert _{\infty}$ with the same quantity (cf. \cite{chen2012sparseclustering}).

On the other hand, under the definition of $c_{1},c_{2}$ and $\epsilon$, we have 
\[
c_{1}\epsilon= b_1\sqrt{\frac{1-t}{tn\log n}}\cdot\frac{2\log^{2}n}{\ellbig}\sqrt{\frac{n}{t(1-t)}}
=b_1\frac{\sqrt{p\log n}}{t\sqrt{t}}\cdot\frac{\log n}{24\ellbig}\sqrt{\frac{t}{p}}
\ge\frac{\log n}{24\ellbig}\sqrt{\frac{t}{p}}
\]
and similarly $c_{2}\epsilon\ge\frac{\log n}{24\ellbig}\sqrt{\frac{t}{p}}$.
It follows that $\left\Vert \PP_{T}Q\right\Vert _{\infty}\le9\cdot\frac{1}{24}\epsilon\min\left\{ c_{1},c_{2}\right\} $,
proving property 2).

\textbf{Properties 3a) and 3b)}

For 3a), by construction of $ Q $ we have 
\begin{eqnarray*}
\left\langle UU^{\top}+Q,\PP_{\Gamma(A)}\PP_{\Gamma}\Pbig\Delta\right\rangle  & = & \left\langle \PP_{\Gamma(A)}\PP_{\Gamma}\Pbig Q_{3},\PP_{\Gamma(A)}\PP_{\Gamma}\Pbig\Delta\right\rangle \\
 & = & (1+\epsilon)c_{1}\sum_{(i,j)\in\Gamma\cap\Gamma(A)}\Pbig\Delta(i,j)\\
 & = & (1+\epsilon)c_{1}\left\Vert \PP_{\Gamma(A)}\PP_{\Gamma}\Pbig\Delta\right\Vert _{1}\mbox{ (because \ensuremath{\Delta\in\DD})}
\end{eqnarray*}
Property 3b) can be verified similarly.

\textbf{Properties 4a) and 4b):}

For 4a), we have
\begin{eqnarray*}
\left\langle UU^{\top}+Q,\PP_{\Gamma(A)}\PP_{\Gamma^{c}}\PP_{\sharp}\Delta\right\rangle  & = & \left\langle \PP_{\Gamma(A)}\PP_{\Gamma^{c}}\PP_{\sharp}\left(UU^{\top}+\Pbig Q_{1}+\Pbig Q_{2}\right),\PP_{\Gamma(A)}\PP_{\Gamma^{c}}\PP_{\sharp}\Delta\right\rangle \\
 & = & \sum_{(i,j)\in\Gamma^{c}\cap\Gamma(A)}\left(\frac{1}{n_{c(i)}}+\frac{1}{n_{c(i)}}\frac{1-p_{ij}}{p_{ij}}+(1+\epsilon)c_{2}\frac{1-p_{ij}}{p_{ij}}\right)\PP_{\sharp}\Delta(i,j)\\
 & \ge & -\left(\frac{1}{p\ellbig}+(1+\epsilon)c_{2}\frac{1-p}{p}\right)\left\Vert \PP_{\Gamma(A)}P_{\Gamma^{c}}\Pbig\Delta\right\Vert _{1},\\
 &  & \mbox{ (here we use \ensuremath{\Delta\in\DD}, \ensuremath{p_{ij}\ge p}, and \ensuremath{n_{c(i)}\ge\ellbig}for \ensuremath{i\in\Vbig})}.
\end{eqnarray*}
Consider the two terms in the parenthesis in the last RHS. For the first term, we have 
\[
\frac{1}{p\ellbig}
=\frac{2\log^{2}n}{\ellbig}\sqrt{\frac{n}{t(1-t)}}\cdot\sqrt{\frac{t(1-t)}{4p^{2}n\log^{4}n}}
\le \frac{2\log^{2}n}{\ellbig}\sqrt{\frac{n}{t(1-t)}}\cdot b_1\sqrt{\frac{1-t}{tn\log n}}
=\epsilon c_{1}.
\]
For the second term, we have the following
\begin{eqnarray*}
p-t\ge\frac{p-q}{4} & \ge & \frac{b_3}{4}\frac{\log^{2}n\sqrt{p(1-q)n}}{\ellbig}\\
 & = & \frac{b_3}{4} \cdot\frac{\sqrt{t(1-q})}{\sqrt{p(1-t)}}\cdot p(1-t)\cdot\frac{2\log^{2}n\sqrt{n}}{\ellbig\sqrt{t(1-t)}}\\
 & \ge & 8\cdot p(1-t)\cdot\frac{2\log^{2}n\sqrt{n}}{\ellbig\sqrt{t(1-t)}}=8p(1-t)\epsilon,
\end{eqnarray*}
which implies $(1+\epsilon)c_{2}\frac{1-p}{p}\le(1-2\epsilon)c_{1}$.
We conclude that 
\[
\left\langle UU^{\top}+Q,\PP_{\Gamma(A)}\PP_{\Gamma^{c}}\PP_{\sharp}\Delta\right\rangle \ge-\left(\epsilon c_{1}+(1-2\epsilon)c_{1}\right)\left\Vert \PP_{\Gamma(A)}P_{\Gamma^{c}}\Pbig\Delta\right\Vert _{1},
\]
proving property 4a). 

For 4b), we have 
\begin{eqnarray*}
\left\langle  UU^{\top}+Q, \PP_{\Gamma(A)^{c}}\PP_{\Gamma^{c}}\PP_{\sharp}\Delta\right\rangle 
& = & \left\langle \PP_{\Gamma(A)^{c}}\PP_{\Gamma^{c}}\PP_{\sharp} Q_3,\PP_{\Gamma(A)^{c}}\PP_{\Gamma^{c}}\PP_{\sharp}\Delta\right\rangle \\
 & = &  \sum_{(i,j) \in \Gamma(A)^{c} \cap\Gamma^{c}\cap\range\Pbig } -(1+\epsilon) \frac{c_{1}q_{ij}}{1-q_{ij}} \Pbig\Delta(i,j)  \\
 & \ge & -(1+\epsilon)\frac{c_{1}q}{1-q}\left\Vert \PP_{\Gamma(A)^{c}}\PP_{\Gamma^{c}}\PP_{\sharp}\Delta \right\Vert _{1}.
 \mbox{ (here we use $ q_{ij}\le q $)}
\end{eqnarray*}
Consider the factor before the norm in the last RHS. Similarly as before, we have
\begin{eqnarray*}
t-q\ge  \frac{p-q}{4}  & \ge &  \frac{b_3}{4}\frac{\log^{2}n\sqrt{p(1-q)n}}{\ellbig}\\
& \ge & 2\cdot t(1-q)\cdot\frac{2\log^{2}n\sqrt{n}}{\ellbig\sqrt{t(1-t)}} = 2 t(1-q) \epsilon.
\end{eqnarray*}
This implies $ (1+\epsilon)c_1\frac{q_{}}{1-q_{}}\le(1-\epsilon)c_2 $. We conclude that 
$$
\left\langle  UU^{\top}+Q, \PP_{\Gamma(A)^{c}}\PP_{\Gamma^{c}}\PP_{\sharp}\Delta\right\rangle 
\ge -(1-\epsilon)c_{2}\left\Vert \PP_{\Gamma(A)^{c}}\PP_{\Gamma^{c}}\PP_{\sharp}\Delta \right\Vert _{1},
$$
proving property 4b).

\textbf{Properties 5) and 6):}  It is obvious that these two properties hold by construction of $ Q $.

Note that properties 3)-6) hold deterministically.

\subsection{The $\kappa>1$ case}

Let $n'=\kappa^{2}n$ and assume $n'$ is an integer. Let $A'\in\mathbb{R}^{n'\times n'}$
be such a matrix that 
\[
A'=\left[\begin{array}{cc}
A & 0\\
0 & I
\end{array}\right].
\]
 Consider the following padded program
\begin{eqnarray*}
\mbox{(CP1') }\min_{K',B'\in\mathbb{R}^{n'\times n'}} &  & \left\Vert K'\right\Vert _{*}+c_{1}\left\Vert \PP_{\Gamma(A')}B'\right\Vert _{1}+c_{2}\left\Vert \PP_{\Gamma(A')^{c}}B'\right\Vert _{1}\\
\mbox{s.t.} &  & K'+B'=A'\\
 &  & 0\le K_{ij}'\le1,\forall(i,j).
\end{eqnarray*}
Applying Theorem \ref{thm:main} with $ \kappa=1 $ (which we have proved) to $A'$ and the padded program
(CP1'), we conclude that the unique optimal solution $(\hat{K}',\hat{B}'=A'-\hat{K}')$
to (CP1') has the form 
\[
\hat{K}'=\left[\begin{array}{cc}
\Pbig K^{*} & 0\\
0 & 0
\end{array}\right].
\]
We claim that $\hat{K}=\Pbig K^{*}$ is the unique optimal solution
to (CP1). 

Proof by contradiction: suppose an optimal solution to (CP1) is $\hat{K}=K_{0}\neq\Pbig K^{*}$.
By optimality we have 
\[
\left\Vert K_{0}\right\Vert _{*}+c_{1}\left\Vert \PP_{\Gamma(A)}(A-K_{0})\right\Vert _{1}+c_{2}\left\Vert \PP_{\Gamma(A)^{c}}(A-K_{0})\right\Vert _{1}\le\left\Vert \Pbig K^{*}\right\Vert _{*}+c_{1}\left\Vert \PP_{\Gamma(A)}(A-\Pbig K^{*})\right\Vert _{1}+c_{2}\left\Vert \PP_{\Gamma(A)^{c}}(A-\Pbig K^{*}\right\Vert _{1}.
\]
Define $K_{0}'=\left[\begin{array}{cc}
K_{0} & 0\\
0 & 0
\end{array}\right]\in\mathbb{R}^{n'\times n'}.$ It follows that
\begin{eqnarray*}
 &  & \left\Vert K_{0}'\right\Vert _{*}+c_{1}\left\Vert \PP_{\Gamma(A')}(A'-K_{0}')\right\Vert _{1}+c_{2}\left\Vert \PP_{\Gamma(A')^{c}}(A'-K_{0}')\right\Vert _{1}\\
 & = & \left\Vert K_{0}\right\Vert _{*}+c_{1}\left\Vert \PP_{\Gamma(A)}(A-K_{0})\right\Vert _{1}+c_{1}(n'-n)+c_{2}\left\Vert \PP_{\Gamma(A)^{c}}(A-K_{0})\right\Vert _{1}\\
 & \le & \left\Vert \Pbig K^{*}\right\Vert _{*}+c_{1}\left\Vert \PP_{\Gamma(A)}(A-\Pbig K^{*})\right\Vert _{1}+c_{1}(n'-n)+c_{2}\left\Vert \PP_{\Gamma(A)^{c}}(A-\Pbig K^{*})\right\Vert _{1}\\
 & = & \left\Vert \hat{K}'\right\Vert _{*}+c_{1}\left\Vert \PP_{\Gamma(A')}(A'-\hat{K}')\right\Vert _{1}+c_{2}\left\Vert \PP_{\Gamma(A')^{c}}(A'-\hat{K}')\right\Vert _{1},
\end{eqnarray*}
contradicting the fact that $(\hat{K}',\hat{B}'=A'-\hat{K}')$ is
the unique optimal to (CP1').


\section{Proof of Theorem~\ref{soundness}}\label{sec:proof:soundness}

Fix $\kappa\geq 1$ and $t$ in the allowed range, let $(K,B)$ be an optimal solution to (CP1) ,
and assume $K$ is a partial clustering induced by $U_1,\dots, U_r$ for some integer $r$, and also assume
$\sigmamin(K) = \min_{i\in [r]} |U_i|$ satisfies (\ref{sigmaminbound}).
Let $M = \sigmamin(K)$.


We need a few helpful facts.  First, note that any value of $t$ in the allowed range $[\frac 1 4 p + \frac 3 4 q, \frac 3 4 p+  \frac 1 4 q]$ 
 satisfies
$q+\frac 1 4(p-q) \leq t \leq p - \frac 1 4(p-q)$.  Also note that from the definition of $t, c_1,c_2$,
\begin{equation}\label{c_ineq}
q + \frac 1 4(p-q) \leq \frac{c_2}{c_1+c_2}=  t \leq p - \frac 1 4(p-q)\ .
\end{equation}

We say that a pair of sets $Y\subseteq V, Z\subseteq V$ is  \emph{cluster separated}  if there is no pair $(y,z)\in Y\times Z$
satisfying $y\sim z$.

\begin{assumption}
There exists a constant $C'>0$ such that for all pairs of cluster-separated sets $Y,Z$ 
of size at least 
$m := \frac {C'\log n}{(p-q)^2}$
each,
\begin{equation}\label{hatq} |\hat d_{Y,Z}-q|  <  \frac 1 4(p-q)\ ,\end{equation}
where $\hat d_{Y,Z}:= \frac{|(Y\times Z) \cap \Omega|} {|Y|\cdot |Z|}$.
\end{assumption}

This is proven by a Hoeffding tail bound and a union bound to hold with probability at least $1-n^{-4}$.  
To see why, fix the sizes $m_Y, m_Z$ of $|Y|,|Z|$, assume $m_Y \leq m_Z$ w.l.o.g.  For each such choice, there are at most $\exp\{C(m_Y+m_Z)\log n\} \leq \exp\{2Cm_Z\log n\}$
possibilities for the choice of sets $Y,Z$, for some $C>0$.  For each such choice, the probability that $(\ref{hatq})$ does not hold is \begin{equation}\label{hoeff1}\exp\{-C'' m_Y m_Z(p-q)^2\}\end{equation}
using Hoeffding inequality, for some $C''>0$.  Hence, as long as $m_Y \geq m$ as defined above, for properly chosen $C'$, using union bound
(over all possibilities of $m_Y, m_Z$ and of $Y, Z$) we obtain (\ref{hatq}) uniformly.

If we assume also , say, that
\begin{equation}\label{Mbym}
M \geq 3m\ ,
\end{equation}
(which can be done by setting $C_1 \geq 3C'$)
 the implication of the assumption  is that it cannot be the case that some $U_i$ contains
a subset $U_i'$ of size in the range $\left[m, |U_i|-m\right ]$ such that
 $U_i' = V_g\cap U_i$ for some $g$.  Indeed, if such a set existed, then we would find a strictly better solution to (CP1),
call it $(K',B')$, which is defined so that $K'$ is obtained from $K$ by splitting the block corresponding to $U_i$
into two blocks, one corresponding to $U_i'$ and the other to $U_i\setminus U_i'$.
The difference $\Delta$ between the cost of $(K,B)$ and $(K',B')$ is (renaming $Y := U_i'$ and $Z := U\setminus U_i'$)   $\Delta = c_1|(Y\times Z)\cap \Omega|  - c_2|(Y\times Z)\cap \Omega^c|
 = (c_1+c_2) \hat d_{Y,Z} |Y|\,|Z| - c_2|Y|\,|Z|$.  But the sign of $\Delta$ is exactly the sign of $\hat d_{Y,Z} - \frac {c_2}{c_1+c_2}$ which
is strictly negative by (\ref{hatq}) and (\ref{c_ineq}).  (We also used the fact that the trace norm part of the utility function
is equal for both solutions: $\|K'\|_*=\|K\|_*$).
 
 The conclusion is that for each $i$, the sets $(U_i\cap V_1),\dots, (U_i\cap V_k)$ must all be of size at most $m$, 
 except maybe for at most one set of size at least $|U_i| - m$.  
If we now also assume that 
\begin{equation}\label{hoeff2}M > k m = (kC'\log n)/(p-q)^2\ ,\end{equation} then 
 we conclude that not all these sets can be of size 
at most $m$.
Hence
 exactly one of these sets must have size at least $|U_i| - m$.
From this we conclude that there is a function $\phi : [r] \mapsto [k]$ such that for all $i\in [r]$,
$$ |U_i\cap V_{\phi(i)}| \geq |U_i| - m\ .$$

\noindent
We now claim that this function is an injection. 
We will need the following assumption:
\begin{assumption}\label{ass5}
 For any $4$ pairwise disjoint subsets $(Y,Y', Z,Z')$ such that $(Y\cup Y')\subseteq V_i$ for some $i$, 
 $(Z \cup Z') \subseteq [n]\setminus V_i$, 
 $\max\{|Z|,|Z'|\} \leq m$, 
 $\min\{|Y|,|Y'|\}\geq M-m$:
 \begin{eqnarray}
|Y|\cdot|Y'|\, \hat d_{Y, Y'} - |Y|\cdot|Z|\, \hat d_{Y,Z} - |Y'|\cdot|Z'|\, \hat d_{Y',Z'} > \nonumber \\
 \frac {c_2}{c_1+c_2}(|Y|\cdot |Y'| - |Y|\cdot|Z| - |Y'|\cdot|Z'|) \label{ass5eq}
 \end{eqnarray}
\end{assumption}
The assumption holds with probability at least $1-n^{-4}$ by using Hoeffding inequality, union bounding over all possible sets $Y,Y',Z,Z'$ as
above.  Indeed, notice that for fixed $m_Y, m_{Y'}, m_Z, m_{Z'}$ (with, say,  $m_Y\geq m_{Y'}$), and for each tuple $Y,Y',Z,Z'$ such that $|Y|=m_Y, |Y'|=m_{Y'}, |Z|=m_Z, |Z'|=m_{Z'}$,
the probability that (\ref{ass5eq}) is violated is at most 
\begin{equation}\label{hoeff3}\exp\{-C(p-q)^2(m_Ym_{Y'} + m_Ym_Z + m_{Y'}m_{Z'})\}\end{equation} for some $C>0$.  Using (\ref{Mbym}), this is at most 
\begin{equation}\label{hoeff4}\exp\{-C''(p-q)^2(m_Y m_{Y'})\}\ ,\end{equation} for some global $C''>0$. Now notice that the number
of possibilities to choose such a $4$ tuple of sets  is bounded above by $\exp\{C''' m_Y\log n\}$, for some global $C'''>0$.
Assuming 
\begin{equation}\label{hoeff5}M \geq \frac{\hat C \log n}{(p-q)^2}\end{equation} for some $\hat C$, and 
applying a union bound over all possible combinations $Y,Y',Z,Z'$ of sizes $m_Y,m_{Y'}, m_Z, m_{Z'}$ respectively, of which there
are at most $\exp\{C^\circ m_Y\log n\}$ for some $C^\circ>0$,
we conclude that (\ref{ass5eq}) is violated for some combination with probability at most 
\begin{equation}\label{hoeff6}\exp\{-C''(p-q)^2m_Y m_{Y'}/2\}\end{equation}
which is at most $\exp\{-20\log n\}$ if 
\begin{equation}\label{hoeff7}M\geq \frac{\hat C' \log n}{(p-q)^2}\ .\end{equation} for some $\hat C'>0$.
Apply a union bound now over the possible combinations of the tuple $(m_Y, m_{Y'}, m_Z, m_{Z'})$ , of which there
are at most $\exp\{4 \log n\}$  to conclude that (\ref{ass5eq}) holds uniformly for all possibilities of $Y,Y',Z,Z'$
with probability at least $1-n^{-4}$.

Now assume by contradiction that $\phi$ is not an injection, so $\phi(i)=\phi(i')=: j$ for some distinct $i,i'\in [r]$.  
Set $Y=U_i\cap V_j, Y' = U_{i'}\cap V_j$, $Z=U_i\setminus Y, Z'=U_{i'}\setminus Y'$.
Note that $\max\{|Z|,|Z'|\} \leq m$ and 
 $\min\{|Y|,|Y'|\}\geq M-m$.
Consider the solution $(K',B')$ where $K'$ is obtained from $K$ by replacing the two blocks corresponding to
$U_i, U_{i'}$ with four blocks: $Y,Y',Z,Z'$.  Inequality (\ref{ass5eq}) guarantees that the cost of $(K',B')$ is
strictly lower than that of $(K,B)$, contradicting optimality of the latter.  (Note that $\|K\|_*=\|K'\|_*$.)

We can now also conclude that $r \leq k$.
Fix $i\in[r]$.  We show that not too many elements of $V_{\phi(i)}$ can be contained in
$V\setminus \{U_1\cup \cdots \cup U_r\}$.
We need the following assumption.
\begin{assumption}\label{ass6}
For all pairwise disjoint sets $Y,X,Z \subseteq V$ such that $|Y| \geq M-m$, $|X| \geq m$,  $(Y\cup  X)\subseteq V_j$ for some $j\in [k]$, $|Z| \leq m$, $Z\cap V_j = \emptyset$:
\begin{eqnarray}
|X|\cdot |Y| \hat d_{X,Y} + {\binom{|X|} {2}} \hat d_{x,x} - |Y|\cdot|Z| \hat d_{Y,Z} >\ \ \ \ \ \ \ \ \ \nonumber \\
 \frac{c_2}{c_1+c_2}(|X|\cdot|Y| + {\binom{|X|} 2} - |Y|\cdot |Z|) + \frac {|X|}{c_1+c_2}\ . \label{ass6eq}
\end{eqnarray}
\end{assumption}
The assumption holds with probability at least $1-n^{-4}$.  To see why, first notice that $|X|/(c_1+c_2) \leq \frac 1 8 (p-q) |X|\cdot |Y|$ by (\ref{sigmaminbound}), as long as $C_2$ is large enough.
This implies that the RHS of (\ref{ass6eq}) is upper
bounded by 
\begin{equation}\label{ass6aux}
\left (p-\frac 1 8(p-q)\right) |X|\cdot |Y| + \frac {c_2}{c_1+c} (\binom{|X|} 2 - |Y|\cdot |Z|)
\end{equation}
Proving that the LHS of (\ref{ass6eq})  (denoted $f(X,Y,Z)$) is larger than (\ref{ass6aux}) (denoted $g(X,Y,Z)$) uniformly w.h.p. can now be easily done as follows.  By fixing $m_Y = |Y|, m_X=|X|$, the number of 
combinations for $Y,X,Z$ is at most  $\exp\{C (m_Y+m_X) \log n\}$ for some global $C>0$.
  On the other hand, the probability that $f(X,Y,Z) \leq g(X,Y,Z)$ for any such option is
at most 
\begin{equation}\label{hoeff8}\exp\{-C'' (p-q)^2 m_Y m_X\}\end{equation} for some $C'>0$.  
Hence, by union bounding, the probability that some tuple $Y,X,Z$ of sizes $m_Y,m_X,m_Z$ respectively
satisfies $f(X,Y,Z) \leq g(X,Y,Z)$ is at most 
\begin{equation}\label{hoeff9}\exp\{-C'' (p-q)^2 m_Y/2\}\ ,\end{equation} which is at most $\exp\{-10\log n\}$
assuming 
\begin{equation}\label{hoeff10} M \geq \bar C (\log n)/(p-q)^2\ ,\end{equation} for some $\bar C>0$.
Another union bound over the possible choices of $m_Y, m_X, m_Z$ proves that (\ref{ass6eq}) holds uniformly with
probability at least $1-n^{-4}$.

Now for some $i\in [r]$ set  $X:= V_{\phi(i)} \cap (V\setminus\{U_1\cup \cdots \cup U_r\})$ and assume by contradiction 
that $|X| > m$.  Set $Y := V_{\phi(i)}\cap U_i$ and $Z = U_i\setminus V_{\phi(i)}$.  Define the solution
 $(K',B')$ where $K'$ is obtained from $K$ by replacing the block corresponding to $U_i$ in $K$ with two blocks: $V_{\phi(i)}$
and $U_i\setminus V_{\phi(i)}$.  Assumption~\ref{ass6} tells us that the cost of $(K',B')$ is strictly lower than that
of $(K,B)$.  Note that the expression $\frac {|X|}{c_1+c_2}$ in the RHS of (\ref{ass6eq}) accounts for the
trace norm difference $\|K'\|_* - \|K\|_* = |X|$.

We are prepared to perform the final ``cleanup'' step.  
At this point we know that for each $i\in [r]$, the set $T_{i} = U_i\cap V_{\phi(i)}$ satisfies
\begin{eqnarray*}
|T_{i}| &\geq& |U_i| -  m \\
|T_i| &\geq& |V_j| - rm \ .
\end{eqnarray*}
\def\profile{\mathcal Y}
(The second inequality is implied by the fact that at most $m$ elements of $V_{\phi(i)}$ may be contained in $U_{i'}$
for $i'\neq i$, and another at most  $m$ elements in $V\setminus (U_1\cup \cdots\cup U_r)$.
We are now going to conclude from this that $U_i = V_{\phi(i)}$ for all $i$.
To that end, let $(K',B')$ be the feasible solution to (CP1) defined so that $K'$ is a partial clustering
induced by $V_{\phi(1)},\dots, V_{\phi(r)}$.  We would like to argue that if $K\neq K'$ then  the cost of $(K',B')$ is strictly smaller than
that of $(K,B)$.   Fix the value of the collection
\begin{eqnarray*}
\profile &:= & ((r, \phi(1),\dots, \phi(r), \\
& &\ \ \ \  \left(m_{ij} := |V_{\phi(i)}\cap U_j|)\right)_{i,j\in [r], i\neq j}\ \ \ , \\
& &  \ \ \ \ \left (m'_i := |V_{\phi(i)} \cap (V\setminus (U_1\cup\cdots\cup U_r))\right)_{i \in [r]})
\end{eqnarray*}  
Let $\beta(\profile)$ denote the number of $i\neq j$ such that $m_{ij}>0$ plus the number of $i\in [r]$ such that $m_i>0$.
We can assume $\beta(\profile)>0$, otherwise $U_i=V_{\phi(i)}$ for all $i\in [r]$ as required.
 The number of possibilities
for $K$ and $K'$ giving rise to $\profile$ is $\exp\{C(\sum_{i\neq j} m_{ij}+ \sum_i m_i)\log n\}$ for some $C>0$.
(Note that $K'$ depends on $r, \phi(1),\dots, \phi(r)$ only, while $K$ depends on all elements of $\profile$).
For each such possibility, the probability that the cost of $(K,B)$ is lower than that of $(K',B')$ is at most
\begin{equation}\label{hoeff11}\exp\{-C''(p-q)^2M (\sum_{ij}m_{ij} + \sum_i m_i)\}\end{equation} using Hoeffding inequalities, for some $C''>0$.  (Note that special
care needs to be made to account for the difference $\|K\|_* - \|K'\|_* = \sum_{i=1}^ r m_i$ - this is similar to what
we did above .)
As long as 
\begin{equation}\label{hoeff12}M \geq \hat C^\dagger k (\log n)/(p-q)^2\end{equation}
 for some $\hat C^\dagger>0$, 
we conclude that the cost of $(K',B')$
is at least that of $(K,B)$ for some $K$ giving rise to  $\profile$ with probability at most $\exp\{-10(k\log n)\beta(\profile)\}$.
  The number of combinations of  $\profile$ for a fixed value of $\beta(\profile)$ is
at most $\exp\{5(k + \beta(\profile)\log n\}$.
By union bounding, we conclude that
for fixed $\beta(\profile)$, the probability that some $(K,B)$ has cost at most that of $(K',B')$ is at most
$\exp\{-10(k\log n)\beta(\profile)\}$.  
Finally union bound over all possibilities
for $\beta(\profile)$, of which there are at most $n^2$.

\noindent
Taking $C_1,C_2$ large enough to satisfy the requirements above concludes the proof.

\section{Proof of Theorem~\ref{soundness2_}}
The proof of Theorem~\ref{soundness} in  the previous section 
made repeated use of Hoeffding tail inequalities, for bounding the size of the intersection of the noise support $\Omega$
with various submatrices.  This is tight for $p,q$ which are bounded away from $0$ and $1$.  However, if $p=\rho p', q=\rho q'$,
the noise probabilities $p',q'$ are fixed and $\rho$ tends to $0$, a sharper bound is obtained using Bernstein tail bound (see Appendix~\ref{sec:bern},Lemma~\ref{lem:bern}).  Using Bernstein inequality instead of Chernoff inequality,
 the expression $(p-q)^2$ in 
(\ref{hoeff1}),(\ref{hoeff2}),(\ref{hoeff3}),(\ref{hoeff4}),(\ref{hoeff5}),(\ref{hoeff6}),(\ref{hoeff7}),(\ref{hoeff8}),(\ref{hoeff9}),
(\ref{hoeff10}),(\ref{hoeff11}) can be replaced with $\rho$.  This clearly gives the required result.

\section{Proof of Lemma~\ref{findkappaalg}}
\begin{proof}
We remind the user that $g = \frac{b_3}{b_4}\log^2 n$, the multiplicative size of the interval $\ellsmall, \ellbig$.
Consider the set of intervals  $(n/gk_0, n/k_0), (n/g^2k_0, n/gk_0),\dots,  (n/g^{k_0+1}k_0 ,n/g^{k_0}k_0)$.
By the pigeonhole principle, one of these intervals must not intersect the set of cluster sizes.
Assume this interval is $(n/g^{i_0+1}k_0, n/g^{i_0} k_0)$, for some $0\leq i_0\leq k_0$.  Let $\alpha = n/g^{i+1}k_0$.  
By setting $C_3(p,q)$ small enough and $C_4(p,q)$ large enough, one easily checks that the
requirements of Corollary~\ref{getoneclust} hold with this value of $\alpha$ and $s=n/k_0$.
This concludes the proof.
\end{proof}



\section{Technical Lemmas}\label{sec:tech_lemma}

\subsection{The spectral norm of random matrices}
It is well-known that the spectral norm $ \lambda_{1}(A) $ of a zero-mean random matrix $ A $ is bounded above w.h.p. by $ C \sqrt{n} $, where $ C $ is a constant that might depend on the variance and magnitude of the entries of $ A $. Here we state and (re-)prove an upper bound of $ \lambda_1(A) $ with an explicit estimate of the constant $ C $, which is needed in the proof of the main theorem.

\begin{lem}
\label{lem:rand_matrix} Let
$A_{ij}$, $1\le i, j\le n$ be independent random variables, each
of which has mean $0$ and variance at most $\sigma^{2}$ and is bounded
in absolute value by $B$. Then with probability at least $ 1-2n^{-2} $
\[
\lambda_{1}(A)\le 6\max\left\{ \sigma\sqrt{n\log n},B\log^{2}n\right\}
\]
\end{lem}
\begin{proof}
Let $ e_i $ be the $ i $-th standard basis in $ \mathbb{R}^n $. Let $Z_{ij} = A_{ij} e_i e_j^\top $. 
Then $ Z_{ij}  $'s are zero-mean random matrices independent of each other, and $ A = \sum_{i,j} Z_{ij} $. 
We have $ \Vert Z_{ij} \Vert \le B $ almost surely. 
We also have $ \Vert \sum_{i,j} \mathbb{E} (Z_{ij} Z_{ij}^\top ) \Vert = \Vert \sum_{i} e_i e_i^\top  \sum_j \mathbb{E} ( A_{ij}^2 ) \Vert \le n\sigma^2.$  
Similarly $ \Vert \sum_{i,j} \mathbb{E} (Z_{ij}^\top Z_{ij} ) \Vert \le n\sigma^2.$ 
Applying the  Non-commutative Bernstein Inequality (Theorem 1.6 in \cite{tropp2010matrixmtg}) with $ t =6\max\left\{ \sigma\sqrt{n\log n},B\log^{2}n\right\} $ 
yields the desired bound.
\end{proof}

\subsection{Standard Bernstein Inequality for Sum of Independent Variables}\label{sec:bern}

\begin{lem}\label{lem:bern}(\emph{Bernstein inequality})
$Let$ $Y_{1},\ldots,Y_{N}$
be independent random variables,  each of which has variance bounded
by $\sigma^{2}$ and is bounded in absolute value by $B$ a.s..
Then we have that 

$$ \Pr\left[ \left |\sum_{i=1}^N Y_i - \mathbb{E}\left[\sum_{i=1}^N Y_i\right]\right |  > t\right ] \leq 2\exp\left \{\frac{t^2/2}{N\sigma^2 + Bt/3}\right\}\ .$$

\end{lem}

\noindent
The following  well known consequence of the theorem will also be of use.
\begin{lem}
\label{lem:subgaussian}(\cite{vershynin2010nonasym}, Proposition 5.16) $Let$ $Y_{1},\ldots,Y_{N}$
be independent random variables, each of which has variance bounded
by $\sigma^{2}$ and is bounded in absolute value by $B$ a.s. Then
we have
\[
\left|\sum_{i=1}^{N}Y_{i}-\mathbb{E}\left[\sum_{i=1}^{N}Y_{i}\right]\right|\le C_{0}\max\left\{ \sigma\sqrt{N\log n},B\log n\right\}
\]
with probability at least $1-C_{1}n^{-C_{2}}$ where the positive constants
$C_{0}$, $C_{1}$, $C_{2}$ are independent of $\sigma$, $B$, $N$
and $n$.
\end{lem}

\end{document}